\documentclass{article} 
\usepackage[parfill]{parskip}
\usepackage{graphicx}
\usepackage{subfig}
\usepackage{algorithm,algorithmic}
\usepackage{amsmath, amssymb, amsthm, bm,bbm}
\usepackage{url}
\usepackage{fullpage, prettyref}
\usepackage{boxedminipage}
\usepackage{hyperref}
\usepackage{wrapfig}
\usepackage{ifthen}
\usepackage{xcolor}
\usepackage{psfrag}
\usepackage{framed}
\usepackage{algorithmic,algorithm}
\usepackage{enumitem}
\usepackage{epstopdf}
\usepackage{mathtools}
\usepackage{caption}

\newtheorem{theorem}{Theorem}

\newtheorem{lemma}{Lemma}

\newtheorem{corollary}{Corollary}
\newtheorem{definition}{Definition}

\newtheorem{assumption}{Assumption}

\newtheorem{remark}{Remark}

\newenvironment{proofof}[1]{\smallskip\noindent{\bf Proof of #1:}}%
        {\hspace*{\fill}$\Box$\par}

\newcommand{\ignore}[1]{}

\newcommand{\R}{\mathbb R}

\newcommand{\abs}[1]{\left\lvert #1 \right\rvert}
\newcommand{\norm}[1]{\left\lVert #1 \right\rVert}
\newcommand{\zeronorm}[1]{\left\lVert #1 \right\rVert_0}
\newcommand{\infnorm}[1]{\left\lVert #1 \right\rVert_{\infty}}

\newcommand{\twonorm}[1]{\left\lVert #1 \right\rVert_{2}}
\newcommand{\frob}[1]{\left\lVert #1 \right\rVert_{F}}

\newcommand{\Sec}[1]{\hyperref[sec:#1]{\S\ref*{sec:#1}}} 
\newcommand{\Eqn}[1]{\hyperref[eq:#1]{(\ref*{eq:#1})}} 
\newcommand{\Fig}[1]{\hyperref[fig:#1]{Fig.\,\ref*{fig:#1}}} 
\newcommand{\Tab}[1]{\hyperref[tab:#1]{Tab.\,\ref*{tab:#1}}} 
\newcommand{\Thm}[1]{\hyperref[thm:#1]{Theorem\,\ref*{thm:#1}}} 
\newcommand{\Lem}[1]{\hyperref[lem:#1]{Lemma\,\ref*{lem:#1}}} 
\newcommand{\Prop}[1]{\hyperref[prop:#1]{Prop.~\ref*{prop:#1}}} 
\newcommand{\Cor}[1]{\hyperref[cor:#1]{Corollary~\ref*{cor:#1}}} 
\newcommand{\Def}[1]{\hyperref[def:#1]{Definition~\ref*{def:#1}}} 
\newcommand{\Alg}[1]{\hyperref[alg:#1]{Alg.~\ref*{alg:#1}}} 
\newcommand{\Ex}[1]{\hyperref[ex:#1]{Ex.~\ref*{ex:#1}}} 
\newcommand{\Clm}[1]{\hyperref[clm:#1]{Claim~\ref*{clm:#1}}} 
\colorlet{shadecolor}{blue!20}

\newcommand{\A}{{A}}

\newcommand{\Pom}{\mathcal{P}_\Omega}
\newcommand{\Pomqt}[1][t]{\mathcal{P}_{\Omega_{q, #1}}}
\newcommand{\Pk}[1][k]{\mathcal{P}_{#1}}
\newcommand{\Sot}[1][t]{\widetilde{S}^{(#1)}}
\newcommand{\HT}{\mathcal{HT}_\zeta}
\newcommand{\Mt}[1][t]{M^{(#1)}}
\newcommand{\Mts}[1][t]{M^{(#1)}_s}
\newcommand{\proj}{P}
\newcommand{\pom}{{\mathcal P}}
\newcommand{\Los}{\Lo_s}
\newcommand{\Hs}{H_s}
\newcommand{\Us}{U_s}

\renewcommand{\L}{L}
\renewcommand{\S}{S}

\newcommand{\Lo}{\L^*}
\newcommand{\So}{\widetilde{\S}^*}
\newcommand{\Sob}{\S^{*}}

\newcommand{\Lt}[1][t]{\L^{(#1)}}
\newcommand{\St}[1][t]{\S^{(#1)}}
\newcommand{\Ltn}{\L^{(t+1)}}
\newcommand{\Stn}{\S^{(t+1)}}
\newcommand{\M}{M}

\newcommand{\E}{{E}}
\newcommand{\Et}[1][t]{\E^{(#1)}}
\newcommand{\Etn}{\E^{(t+1)}}

\newcommand{\ind}[1]{\mathbbm{1}_{\{#1\}}}
\newcommand{\order}[1]{O\left({#1}\right)}
\newcommand{\ncralgo}{\textbf{R-RMC}}
\newcommand{\ncgca}{\textbf{PG-RMC}\ }
\newcommand{\splitS}{\textbf{SplitSamples}}
\newcommand{\supp}[1]{\textrm{Supp}\left(#1\right)}

\newcommand\numberthis{\addtocounter{equation}{1}\tag{\theequation}}
\newcommand{\wSt}{\widetilde{S}^{(t)}}
\newcommand{\wSo}{\widetilde{S}^*}
\newcommand{\thresh}{\eta}
\newcommand{\Lts}[1][t]{\L^{(#1)}_s}

\newcommand{\ddfrac}[2]{\frac{\displaystyle #1}{\displaystyle #2}}

\title{Nearly-optimal Robust Matrix Completion}

\author{Yeshwanth Cherapanamjeri \qquad Kartik Gupta \qquad Prateek Jain \\
Microsoft Research India\\
\texttt{\{t-yecher,t-kagu,prajain\}@microsoft.com}
}

\begin{document}
\maketitle


\begin{abstract}
In this paper, we consider the problem of Robust Matrix Completion (RMC) where the goal is to recover a low-rank matrix by observing a small number of its entries out of which a few can be {\em arbitrarily} corrupted. We propose a simple projected gradient descent method to estimate the low-rank matrix that alternately performs a projected gradient descent step and cleans up a few of the corrupted entries using hard-thresholding. Our algorithm solves RMC using nearly optimal number of observations as well as nearly optimal number of corruptions. Our result also implies significant improvement over the existing time complexity bounds for the low-rank matrix completion problem. Finally, an application of our result to the robust PCA problem (low-rank+sparse matrix separation)  leads to nearly {\em linear} time (in matrix dimensions) algorithm for the same; existing state-of-the-art methods require quadratic time. Our empirical results corroborate our theoretical results and show that even for moderate sized problems, our method for robust PCA is an order of magnitude faster than the existing methods.
\end{abstract}
\section{Introduction}
In this paper, we study the Robust Matrix Completion (RMC) problem where the goal is to recover an underlying low-rank matrix by observing a small number of sparsely corrupted entries from the matrix. Formally,
\begin{equation}
  \label{eq:rmc}
\text{RMC:}\qquad \text{Find rank-}r \ \text{matrix}\ \ \Lo\in \R^{m\times n} \quad \text{ using }\ \  \Omega\ \ and\ \ \Pom(\Lo)+\Sob,
\end{equation}
where $\Omega\subseteq [m]\times[n]$ is the set of observed entries (throughout the paper we assume that $m\leq n$), $\Sob$ denotes the sparse corruptions of the observed entries, i.e., $Supp(\Sob)\subset \Omega$. Sampling operator $\Pom:\R^{m\times n}\rightarrow \R^{m\times n}$ is defined as:
\begin{equation}
  \label{eq:pom}
  (\Pom(A))_{ij} =
  \begin{cases}
    A_{ij}, &\text{ if }\ (i,j)\in \Omega \\
    0, &\text{ otherwise}.
  \end{cases}
\end{equation}

RMC is an important problem with several applications such as recommendation systems with outliers. Similarly, the problem is also heavily used to model PCA under gross outliers as well as erasures \cite{DBLP:journals/jmlr/JalaliRVS11}. Finally, as we show later, an efficient solution to RMC enables faster solution for the robust PCA (RPCA) problem as well. The goal in RPCA is to find a low-rank matrix $\Lo$ and sparse matrix $\So$ by observing their sum, i.e., $M=\Lo+\So$. State-of-the-art results for RPCA shows exact recovery of a rank-$r$, $\mu$-incoherent $\Lo$ (see Assumption 1, Section~\ref{sec:analysis}) if at most $\rho=\frac{1}{\mu^2 r}$ fraction of the entries in each row/column of $\So$ are corrupted \cite{hsu2011robust,NIPS2014_5430}.

However, the existing state-of-the-art results for RMC with optimal $\rho=\frac{1}{\mu^2r}$ fraction of corrupted entries, either require at least a constant fraction of the entries of $\Lo$ to be observed \cite{DBLP:conf/isit/ChenJSC11,CandesLMW11} or require restrictive assumptions like support of corruptions $\So$ being uniformly random \cite{Li2013}. \cite{klopp2014robust} also considers RMC problem but studies the noisy setting and do not provide exact recovery bounds. Moreover, most of the existing methods for RMC use convex relaxation for both low-rank and sparse components, and in general exhibit large time complexity ($O(m^2n)$).



In this work, we attempt to answer the following open question (assuming $m\leq n$): \\[5pt]
{\em Can RMC be solved exactly by using $|\Omega|=O(rn\log n)$ observations out of which $O(\frac{1}{\mu^2 r})$ fraction of the observed entries in each row/column are corrupted.}\\[5pt]
Note that both  $|\Omega|$ (for uniformly random $\Omega$) and $\rho$ values mentioned in the question above denote the information theoretic limits. Hence, the goal is to solve RMC for nearly-optimal number of samples and nearly-optimal fraction of corruptions.

Under standard assumptions on $\Lo,\ \Sob,\ \Omega$ and for $n=O(m)$, we answer the above question in affirmative albeit with $|\Omega|$ which is $O(r)$ (ignoring log factors) larger than the optimal sample complexity (see Theorem~\ref{thm:gca}).  In particular, we propose a simple projected gradient (PGD) style method for RMC that alternately cleans up corrupted entries by hard-thresholding; our method's computational complexity is also nearly optimal ($O(|\Omega|r+(m+n)r^2+r^3)$). 
Our algorithm is based on projected gradient descent for estimating $\Lo$ and alternating projection on set of sparse matrices for estimating $\Sob$. Note that projection is onto non-convex sets of low-rank matrices (for $\Lo$) and sparse matrices (for $\Sob$), hence
standard convex analysis techniques cannot be used for our algorithm.

In {\em a concurrent and independent work}, \cite{DBLP:journals/corr/YiPCC16} also studied the RMC problem and obtained similar results. They study an alternating minimization style algorithm while our algorithm is based on low-rank projected gradient descent. Moreover, our sample complexity,  corruption complexity as well as time complexity differs along certain critical parameters: a) Sample complexity: Our sample complexity bound is dependent only logarithmically on  $\kappa$, the condition number of the matrix $L^*$ (see Table ~\ref{tab:mc}). On the other hand, result of \cite{DBLP:journals/corr/YiPCC16} depends quadratically on $\kappa$, which can be significantly large. However, our sample complexity bound depends logarithmically on the final error $\epsilon$ (defined as $\epsilon = \twonorm{L - \Lo}$); this implies that for typical finite precision computation, our sample complexity bound can be worse by a constant factor. b) Our result allows the fraction of corrupted entries to be information theoretic optimal (up to a constant) $O(\frac{1}{\mu^2 r})$, while the result of \cite{DBLP:journals/corr/YiPCC16}  allows only $O\left(\min\left(\frac{1}{\mu^2r\sqrt{r\kappa}}, \frac{1}{\mu^2\kappa^2r}\right)\right)$ fraction of corrupted entries. c) As a consequence of the sample complexity bounds, running time of the method by \cite{DBLP:journals/corr/YiPCC16} depends quintically on $\kappa$. On the other hand, our algorithm has optimal sparsity (up to a constant factor) independent of $\kappa$ and polylogarithmic dependence on $\kappa$ for sample and running time complexities.

Several recent results \cite{DBLP:conf/colt/0002N15,NIPS2014_5430,DBLP:conf/nips/0002TK14,HardtW14,DBLP:journals/tit/Blumensath11} show that under certain assumptions, projection onto non-convex sets indeed lead to provable algorithms with fast convergence to the global optima. However, as explained in Section~\ref{sec:analysis}, RMC presents unique set of challenges as we have to perform error analysis with the errors arising due to missing entries as well as sparse corruptions, both of which interact among themselves as well. In fact, our careful error analysis also enables us to improve results for the matrix completion as well as the RPCA problem.

{\bf Matrix Completion (MC)}: The goal of MC is to find rank-$r$ $\Lo$ using $\Pom(\Lo)$. State-of-the-art result for MC uses nuclear norm minimization and requires $|\Omega|\geq \mu^2 nr^2 \log^2 n$ under standard $\mu$-incoherence assumption (see Section~\ref{sec:analysis}), but the method requires $O(m^2n)$ time in general. The best sample complexity result for a non-convex iterative method (with at most logarithmic dependence on the condition number of $\Lo$) achieve exact recovery when $|\Omega|\geq \mu^6 n r^5 \log^2 n$ and needs $O(|\Omega|r)$ computational steps. In contrast, assuming $n=O(m)$, our method achieves nearly the same sample complexity of trace-norm but with nearly linear time algorithm ($O(|\Omega|r)$). See Table~\ref{tab:mc} for a detailed comparison of our result with the existing methods.

{\bf RPCA}: Several recent results show that RPCA can be solved if $\rho=O(\frac{1}{\mu^2 r})$-fraction of entries in each row and column of $\Lo$ are corrupted \cite{NIPS2014_5430, hsu2011robust} where $\Lo$ is assumed to be $\mu$-incoherent. Moreover, St-NcRPCA algorithm \cite{NIPS2014_5430} can solve the problem in time $O(mnr^2)$. Corollary~\ref{cor:rpca} shows that by sampling $\Omega$ uniformly at random, we can solve the problem in time $O(nr^3)$ only. That is, we can recover $\Lo$ without even observing the entire input matrix. Moreover, if the goal is to recover the sparse corruption as well, then we can obtain a two-pass (over the input matrix) algorithm that solves the RPCA problem exactly. St-NcRPCA algorithm requires $r^2 \log(1/\epsilon)$ passes over the data. Our method has significantly smaller space complexity as well.

Our empirical results on synthetic data demonstrates effectiveness of our method. We also apply our method to the foreground background separation problem; our method is an order of magnitude faster than the state-of-the-art method (St-NcRPCA) while achieving similar accuracy. 

In summary, this paper's main contributions are:\\[2pt]
{\bf (a) RMC:} We propose a nearly linear time method that solves RMC with $|\Omega|=O(nr^2\log^2 n\log^2 \|M\|_2/\epsilon)$  random entries and with optimal fraction of corruptions ($\rho=\frac{1}{\mu^2 r}$). \\[2pt]
{\bf (b) Matrix Completion:} Our result improves upon the existing linear time algorithm's sample complexity by an $O(r^3)$ factor, and time complexity by $O(r^4)$ factor, although with an extra $O(\log \|\Lo\|/\epsilon)$ factor in both time and sample complexity.\\[2pt] 
{\bf (c) RPCA:} We present a nearly linear time ($O(nr^3)$) algorithm for RPCA under optimal fraction of corruptions, improving upon $O(mnr^2)$ time complexity of the existing methods.

{\bf Notations}: We assume that $M=\Lo+\So$ and $\Pom(M)=\Pom(\Lo)+\Sob$, i.e., $\Sob=\Pom(\So)$. $\|v\|_p$ denotes $\ell_p$ norm of a vector $v$; $\|v\|$ denotes $\ell_2$ norm of $v$. $\|A\|_2$, $\|A\|_F$, $\|A\|_*$ denotes the operator, Frobenius, and nuclear norm of $A$, respectively; by default  $\|A\|=\|A\|_2$. 
Operator $\Pom$ is given by \eqref{eq:pom}, operators $\proj_k(A)$ and $\HT(A)$ are defined in Section~\ref{sec:prob}. $\sigma_i(A)$ denotes $i$-th singular value of $A$ and $\sigma_i^*$ denotes the $i$-th singular value of $\Lo$.

{\bf Paper Organization}: We present our main algorithm in Section~\ref{sec:prob} and our main results in Section~\ref{sec:analysis}. We also present an overview of the proof in Section~\ref{sec:analysis}. Section~\ref{sec:exp} presents our empirical result. Due to lack of space, we present most of the proofs and useful lemmas in Appendix.

\section{Algorithm}\label{sec:prob}
In this section we present our algorithm for solving the RMC (Robust Matrix Completion) problem: given $\Omega$ and $P_{\Omega}(M)$ where $M=\Lo+\So\in \R^{m\times n}$, $rank(\Lo)\leq r$, $\|\So\|_0\leq s$ and $\Sob=P_{\Omega}(\So)$, the goal is to recover $\Lo$. To this end, we focus on solving the following non-convex optimization problem:

\begin{equation}
  \label{eq:rpca_form}
  (\Lo, \Sob)=\arg\min_{L, S}\|P_\Omega(M)-P_\Omega(L)-\S\|_F^2 \ s.t.\ rank(L)\leq r, \Pom(S)=S, \|S\|_0\leq s.
\end{equation}

For the above problem, we propose a simple iterative algorithm that combines projected gradient descent (for $L$) with alternating projections (for $S$). In particular, we maintain iterates $\Lt$ (with rank $k\leq r$) and sparse $\St$. $\Ltn$ is computed using gradient descent step for objective \eqref{eq:rpca_form} and then projecting back onto the set of rank $k$ matrices. That is,
\begin{equation}\Ltn=\proj_k\left(\Lt+\frac{1}{p}\pom_{\Omega}(M-\Lt-\St)\right),\label{eq:lupdate}\end{equation}
where $\proj_k(A)$ denotes projection of $A$ onto the set of rank-$k$ matrices and can be computed efficiently using SVD of $A$, $p= \frac{|\Omega|}{mn}$.  $\Stn$ is computed by projecting the residual $\pom_{\Omega}(M-\Ltn)$ onto set of sparse matrices using a hard-thresholding operator, i.e.,
\begin{equation}\Stn=\HT(M-\Ltn), \label{eq:supdate}\end{equation}
where $\HT:\R^{m\times n}\rightarrow \R^{m\times n}$ is the hard thresholding operator defined as: $(\HT(A))_{ij}=A_{ij}$ if $|A_{ij}|\geq \zeta$ and $0$ otherwise. Intuitively, a better estimate of the sparse corruptions for each iteration will reduce the noise of the projected gradient descent step and a better estimate of the low rank matrix will enable better estimation of the sparse corruptions. Hence, under correct set of assumptions, the algorithm should recover $\Lo$, $\So$ exactly.

Unfortunately, just the above two simple iterations cannot handle problems where $L^*$ has poor condition number, as the intermediate errors can be significantly larger than the smallest singular values of $L^*$, making recovery of the corresponding singular vectors challenging. To alleviate this issue, we propose an algorithm that proceeds in stages. In the $q$-th stage, we project $\Lt$ onto set of rank-$k_q$ matrices. Rank $k_q$ is monotonic w.r.t. $q$. Under standard assumptions, we show that we can increase $k_q$ in a manner such that after each stage $\infnorm{\Lt-\Lo}$  decreased by at least a constant factor. Hence, the number of stages is only logarithmic in the condition number of $\Lo$.

See Algorithm~\ref{alg:gca} (\ncgca) for a pseudo-code of the algorithm. We require an upper bound of the first singular value for our algorithm to work. Specifically, we require $\sigma = \order{\sigma^*_1}$. Alternatively, we can also obtain an estimate of $\sigma^*_1$ by using the thresholding technique from \cite{DBLP:journals/corr/YiPCC16} although this requires an estimate of the number of corruptions in each row and column. We also use a simplified version of Algorithm 5 from \cite{HardtW14} to form independent sets of samples for each iteration which is required for our theoretical analysis. Our algorithm has an ``outer loop'' (see Line 6) which sets rank $k_q$ of iterates $\Lt$ appropriately (see Line 7). We then update $\Lt$ and $\St$ in the ``inner loop'' using \eqref{eq:lupdate}, \eqref{eq:supdate}. We set threshold for the hard-thresholding operator using singular values of current gradient descent update (see Line 12). Note that, we divide $\Omega$ uniformly into $Q\cdot T$ sets, where $Q$ is an upper bound on the number of outer iterations and $T$ is the number of inner iterations. This division ensures independence across iterates that is critical to application of standard concentration bounds; such division is a standard technique in the matrix completion related literature \cite{DBLP:conf/colt/0002N15, HardtW14, DBLP:journals/jmlr/Recht11}. Also, $\thresh$ is a tunable parameter which should be less than one and is smaller for ``easier'' problems.

Note that updating $\St$ requires $O(|\Omega|\cdot r+(m+n)\cdot r)$ computational steps. Computation of $\Ltn$ requires computing SVD for projection $\proj_r$, which can be computed in time $O(|\Omega|\cdot r+(m+n)\cdot r^2+r^3)$ time (ignoring $\log$ factors); see \cite{JainMD10} for more details. Hence, the computational complexity of each step of the algorithm is linear in $|\Omega|\cdot r$ (assuming $|\Omega|\geq r\cdot (m+n)$). As we show in the next section, the algorithm exhibits geometric convergence rate under standard assumptions and hence the overall complexity is still nearly linear in $|\Omega|$ (assuming $r$ is just a constant).

{\bf Rank based Stagewise algorithm}: We also provide a rank-based stagewise algorithm (\ncralgo) where the outer loop increments $k_q$ by one at each stage, i.e., the rank is $q$ in the $q$-th stage. Our analysis extends for this algorithm as well, however, its time and sample complexity trades off a factor  of $O(\log(\sigma_1/\epsilon))$ from the complexity of \ncgca with a factor of $r$ (rank of $\Lo$). We provide the detailed algorithm in Appendix~\ref{app:rank} due to lack of space (see Algorithm~\ref{alg:sap}).

\begin{algorithm}[t!]
  \caption{$\widehat{L} ~ =$ \ncgca$(\Omega, \Pom(M), \epsilon, r, \mu, \thresh, \sigma)$}
  \label{alg:gca}
  \begin{algorithmic}[1]
    \STATE {\bf Input}: Observed entries $\Omega$, Matrix $\Pom(M) \in \R^{m\times n}$, convergence criterion $\epsilon$, target rank $r$, incoherence parameter $\mu$, thresholding parameter $\thresh$, estimate of first singular value $\sigma$
    \STATE  $T\leftarrow 10\log{\frac{10\mu^2 r \sigma}{\epsilon}}$, $Q\leftarrow T$
    \STATE Partition $\Omega$ into $Q\cdot T + 1$ subsets $\{\Omega_0\} \cup \{\Omega_{q, t} : q\in[Q], t\in[T]\}$ using algorithm \ref{alg:split}
    \STATE $L^{(0)}=0,\ \zeta\leftarrow \thresh \sigma$
    \STATE $M^{(0)}=\frac{mn}{\abs{\Omega_0}}\mathcal{P}_{\Omega_0}(M - \HT(M))$
    \STATE {$k_{0} \leftarrow 0$}, {$q \leftarrow 0$}
    \WHILE{$\sigma_{k_{q}+1}(M^{(0)})>\frac{\epsilon}{2\thresh n}$}
    \STATE $q \leftarrow q + 1$,
    \STATE $k_{q}\leftarrow\abs{\{i:\sigma_i(M^{(0)}) \geq\frac{\sigma_{k_{q-1}+1}(M^{(0)})}{2}\}}$
    \FOR{Iteration $t=0$ to $t=T$}
    \STATE $S^{(t)}=\HT(\mathcal{P}_{\Omega_{q, t}}(M-L^{(t)}))$
    \STATE $M^{(t)} = L^{(t)}-\frac{mn}{\abs{\Omega_{q, t}}}\mathcal{P}_{\Omega_{q, t}}( L^{(t)} + S^{(t)} - M)$
    \STATE $L^{(t+1)}=P_{k_{q}}(M^{(t)})$
    \STATE $\zeta\leftarrow \thresh\, \left(\sigma_{k_{q}+1}(M^{(t)}) +\left(\frac{1}{2}\right)^{t - 2} \sigma_{k_{q}}(M^{(t)})\right)\label{eqn:threshold}$
    \ENDFOR
    \STATE $S^{(0)}=S^{(T)}$, $L^{(0)}=L^{(T)}$, $M^{(0)}=M^{(T)}$
    \ENDWHILE
    \STATE {\bf Return: }$L^{(T)}$
  \end{algorithmic}
\end{algorithm}

\begin{algorithm}[t!]
  \caption{$\{\Omega_{1},\dots,\Omega_{T}\} =$ \splitS$(\Omega, p, T)$}
  \label{alg:split}
  \begin{algorithmic}[1]
    \STATE {\bf Input}: Random samples with probability $Tp$ $\Omega$, Required Sampling Probability $p$, Number of Sets $T$
    \STATE {\bf Output}: $T$ independent sets of entries $\{\Omega_{1},\dots,\Omega_{T}\}$ sampled with sampling probability $p$
    \STATE $p' \leftarrow 1 - (1 - p)^T$
    \STATE $\Omega'$ be sampled from $\Omega$ with each entry being included independently with probability $p'/p$

    \FOR{$r = 1$ to $r = T$}
    \STATE $q_r \leftarrow \binom{T}{r} / (2^T - 1)$
    \ENDFOR

    \STATE Initialize $\Omega_{t} \leftarrow \{\}$ for $t \in \{1,\dots,T\}$

    \FOR{Sample $s \in \Omega'$}
    \STATE Draw $r \in \{1,\dots,T\}$ with probability $q_r$
    \STATE Draw a random subset $S$ of size $r$ from $\{1, \dots, T\}$
    \STATE Add $s$ to $\Omega_i \text{ for } i \in S$
    \ENDFOR
  \end{algorithmic}
\end{algorithm}


\renewcommand{\u}{\bm{u}}
\section{Analysis}\label{sec:analysis}
We now present our analysis for both of our algorithms \ncgca\ (Algorithm~\ref{alg:gca}) and \ncralgo\ (Algorithm~\ref{alg:sap}). In general the problem of Robust PCA with Missing Entries \eqref{eq:rpca_form} is harder than the standard Matrix Completion problem and hence is NP-hard \cite{DBLP:conf/colt/HardtMRW14}. Hence, we  need to impose certain (by now standard) assumptions on $\Lo$, $\So$, and $\Omega$ to ensure tractability of the problem:

\begin{assumption}
\label{as:rank}
\textbf{Rank and incoherence of $\Lo$: } $\Lo\in \R^{m\times n}$ is a rank-$r$ incoherent matrix, i.e., $\twonorm{e_{i}^{\top}U^*} \leq \mu \sqrt{\frac{r}{m}}$, $\twonorm{e_{j}^{\top}V^*} \leq \mu \sqrt{\frac{r}{n}}$, $\forall i\in [m],\ \forall j \in [n]$, where $\Lo=U^*\Sigma^* (V^*)^{\top}$ is the SVD of $\Lo$.
\end{assumption}


\begin{assumption}
\label{as:samp}
\textbf{Sampling ($\Omega$): } $\Omega$ is obtained by sampling each entry with probability $p = \frac{\mathbb{E}[|\Omega|]}{mn}$.
\end{assumption}

\begin{assumption}
\label{as:sp}
\textbf{Sparsity of $\So$, $\Sob$: } We assume that at most $\rho\leq \frac{c}{\mu^2 r}$ fraction of the elements in each row and column of $\So$ are non-zero for a small enough constant $c$. Moreover, we assume that $\Omega$ is independent of $\So$. Hence, $\Sob=\Pom(\So)$ also has at most $p \cdot \rho$ fraction of the entries in expectation.
\end{assumption}

Assumptions 1, 2 are standard assumptions in the provable matrix completion literature \cite{CandesR2007, DBLP:journals/jmlr/Recht11, DBLP:conf/colt/0002N15}, while Assumptions 1, 3 are standard assumptions in the robust PCA (low-rank+sparse matrix recovery) literature \cite{ChandrasekaranSPW11, CandesLMW11, hsu2011robust}. Hence, our setting is a generalization of both the standard and popular problems and as we show later in the section, our result can be used to meaningfully improve the state-of-the-art for both these problems.



We first present our main result for Algorithm \ref{alg:gca} under the assumptions given above.
\begin{theorem}
\label{thm:gca}
Let Assumptions 1, 2 and 3 on $\Lo$, $\So$ and $\Omega$ hold respectively. Let $m\leq n$, $n=O(m)$, and  let the number of samples $|\Omega|$ satisfy:
$$\mathbb{E}[|\Omega|] \geq C\alpha\mu^{4}r^{2}n\log^2 \left(n\right) \log^2 \left(\frac{\mu^2r\sigma_{1}}{\epsilon}\right),$$
where $C$ is a global constant. Then, with probability at least $1 - n^{-\log \frac{\alpha}{2}}$, Algorithm \ref{alg:gca} with $\eta = \frac{4\mu^2 r}{m}$, at most $O(\log(\|M\|_2/\epsilon)))$ outer iterations and $O(\log(\frac{\mu^2 r \|M\|_2}{\epsilon}))$ inner iterations, outputs a matrix $\hat{\L}$ such that:
$$\frob{\hat{\L}-\Lo} \leq \epsilon.$$
\end{theorem}
Note that our number of samples increase with the desired accuracy $\epsilon$. However, using argument similar to that of \cite{DBLP:conf/colt/0002N15}, we should be able to replace $\epsilon$ by $\sigma_{\min}^*$ which should modify the $\epsilon$ term to be $\log^2 \kappa$ where  $\kappa=\sigma_1(\Lo)/\sigma_r(\Lo)$. We leave ironing out the details for future work. 

Note that the number of samples matches information theoretic bound upto $O(r \log n\log^2 \sigma_1^*/\epsilon)$ factor. Also, the number of allowed corruptions in $\So$ also matches the known lower bounds (up to a constant factor) and cannot be improved upon information theoretically.

We now present our result for the rank based stagewise algorithm (Algorithm~\ref{alg:sap}). 

\begin{theorem}
\label{thm:sap}
Under Assumptions 1, 2 and 3 on $\Lo$, $\So$ and $\Omega$ respectively and $\Omega$ satisfying:
$$\mathbb{E}[|\Omega|] \geq C\alpha\mu^{4}r^{3}n\log^2 \left(n\right) \log \left(\frac{\mu^2r\sigma_{1}}{\epsilon}\right),$$
for a large enough constant $C$, then Algorithm \ref{alg:sap} with $\thresh$ set to $\frac{4\mu^2r}{m}$ outputs a matrix  $\hat{\L}$ such that: $\frob{\hat{\L}-\Lo} \leq \epsilon, $ w.p. $\geq 1 - n^{-\log \frac{\alpha}{2}}$.
\end{theorem}
Notice that the sample complexity of Algorithm~\ref{alg:sap} has an additional multiplicative factor of $O(r)$ when compared to that of Algorithm~\ref{alg:gca}, but shaves off a factor of $O(\log(\kappa))$. Similarly, computational complexity of Algorithm~\ref{alg:sap} also trades off a $O(\log\kappa)$ factor for $O(r)$ factor from the computational complexity of Algorithm~\ref{alg:gca}.

{\bf Result for Matrix Completion}: Note that for $\So=0$, the RMC problem with Assumptions 1,2 is exactly the same as the standard matrix completion problem and hence, we get the following result as a corollary of Theorem~\ref{thm:gca}:
\begin{corollary}[Matrix Completion]
Suppose we observe $\Omega$ and $P_{\Omega}(\Lo)$ where Assumptions 1,2 hold for $\Lo$ and $\Omega$. Also, let $E[|\Omega|]\geq C\alpha^2 \mu^4r^2 n\log^2n \log^2 \sigma_1/\epsilon$ and $m\leq n$. Then, w.p. $\geq 1-n^{-\log\frac{\alpha}{2}}$, Algorithm~\ref{alg:gca} outputs $\hat{L}$ s.t. $\|\hat{L}-\Lo\|_2\leq \epsilon$.
\end{corollary}
Table~\ref{tab:mc} compares our sample and time complexity bounds for low-rank MC. Note that our sample complexity is nearly the same as that of nuclear-norm methods while the running time of our algorithm is significantly better than the existing results that have at most logarithmic dependence on the condition number of $\Lo$. 

{
\renewcommand{\arraystretch}{1.5}
\begin{table*}[ht]
  \begin{center}
    \vspace{0.3in}
    \caption{Comparison of \ncgca\ and \ncralgo\ with Other Matrix Completion Methods}
    \label{tab:mc}
    \begin{tabular}{| c | c | c |}
      \hline
      & Sample Complexity & Computational Complexity \\ \hline
      Nuclear norm  \cite{DBLP:journals/jmlr/Recht11}		& $\order{\mu^2 r n \log^2 n}$ & $\order{n^3 \log \frac{1}{\epsilon}}$ \\ \hline
      SVP \cite{DBLP:conf/colt/0002N15} & $\order{\mu^4 r^5 n \log^3 n}$ &  $\order{\mu^4 r^7 n \log^3 n \log(\frac{1}{\epsilon})}$ \\ \hline
      Alt. Min. \cite{HardtW14} & $\order{n \mu^4 r^{9} \log^3 \left( \kappa \right) \log^2 n}$ & $\order{n \mu^4 r^{13} \log^3 \left( \kappa \right) \log^2 n}$ \\ \hline
      Alt. Grad. Desc. \cite{SunL15} & $\order{n r \kappa^{2} \max\{\mu^2 \log n, \mu^{4} r^{6} \kappa^{4} \} }$ & $\order{n^2 r^{6} \kappa^{4} \log \left(\frac{1}{\epsilon}\right)}$ \\ \hline
      \ncralgo\ (This Paper) & $\order{\mu^{4}r^{3}n\log^2 \left(n\right) \log \left(\frac{\sigma^{*}_{1}}{\epsilon}\right)}$ & $\order{\mu^{4}r^{4}n\log^2 \left(n\right) \log \left(\frac{\sigma^{*}_{1}}{\epsilon}\right)}$ \\ \hline
      \ncgca\ (This Paper) & $\order{\mu^{4}r^{2}n\log^2 \left(n\right) \log^2 \left(\frac{\sigma^{*}_{1}}{\epsilon}\right)}$ & $\order{\mu^{4}r^{3}n\log^2 \left(n\right) \log^2 \left(\frac{\sigma^{*}_{1}}{\epsilon}\right)}$ \\ \hline
    \end{tabular}
    \vspace{0.1in}
  \end{center}
\end{table*}
}

{\bf Result for Robust PCA}: Consider the standard Robust PCA problem (RPCA), where the goal is to recover $\Lo$ from $M=\Lo+\So$. For RPCA as well, we can randomly sample $|\Omega|$ entries from $M$, where $\Omega$ satisfies the assumption required by Theorem~\ref{thm:gca}. This leads us to the following corollary:
\begin{corollary}[Robust PCA]\label{cor:rpca}
Suppose we observe $M=\Lo+\So$, where Assumptions 1, 3 hold for $\Lo$ and $\So$. Generate $\Omega\in [m]\times [n]$ by sampling each entry uniformly at random with probability $p$, s.t., $E[|\Omega|]\geq C\alpha^2 \mu^4r^2 n\log^2n \log^2 \sigma_1/\epsilon$. Let $m\leq n$. Then, w.p. $\geq 1-n^{-\log\frac{\alpha}{2}}$, Algorithm~\ref{alg:gca} outputs $\hat{L}$ s.t. $\|\hat{L}-\Lo\|_2\leq \epsilon$.
\end{corollary}
Hence, using Theorem~\ref{thm:gca}, we will still be able to recover $\Lo$ but using only the sampled entries. Moreover, the running time of the algorithm is only $O(\mu^2 nr^3 \log^2n \log^2(\sigma_1/\epsilon))$, i.e., we are able to solve RPCA problem in time {\em linear} in $n$. To the best of our knowledge, the existing state-of-the-art methods for RPCA require at least $O(n^2r)$ time to perform the same task \cite{NIPS2014_5430, Wangaistats}. Similarly, we don't need to load the entire data matrix in memory, but we can just sample the matrix and work with the obtained sparse matrix with at most linear number of entries. Hence, our method significantly reduces both time and space complexity, and as demonstrated empirically in Section~\ref{sec:exp} can help scale our algorithm to very large data sets without losing accuracy.
\subsection{Proof Outline for Theorem \ref{thm:gca}}
We now provide an outline of our proof for Theorem \ref{thm:gca} and motivate some of our proof techniques; the proof of Theorem \ref{thm:sap} follows similarly. Recall that we assume that $M=\Lo+\So$ and define $\Sob=\Pom(\So)$. Similarly, we define $\Sot=\HT(M-\Lt)$. {\em Critically}, $\St=\Pom(\Sot)$ (see Line 9 of Algorithm~\ref{alg:gca}), i.e., $\Sot$ is the set of iterates that we ``could'' obtain if entire $M$ was observed. Note that we cannot compute $\Sot$, it is introduced only to simplify our analysis.

We first re-write the projected gradient descent step for $\Ltn$ as described in \eqref{eq:lupdate}:
\begin{equation}
    \label{eqn:rUpdate}
    L^{(t + 1)} = \mathcal{P}_{k_{q}}\Big(\Lo + \underbrace{(\So - \Sot)}_{E_{1}} +
    \underbrace{\left(\mathcal{I} - \frac{\mathcal{P}_{\Omega_{q,t}}}{p}\right)(\overbrace{(L^{(t)} - \Lo)}^{E_{2}} + (\Sot - \So))}_{E_{3}} \Big)
\end{equation}
That is, $\Ltn$ is obtained by rank-$k_q$ SVD of a perturbed version of $\Lo$: $\Lo+E_1+E_3$. As we perform entrywise thresholding to reduce $\|\So-\Sot\|_\infty$, we need to bound $\|\Ltn-\Lo\|_\infty$. To this end, we use techniques from \cite{DBLP:conf/colt/0002N15}, \cite{NIPS2014_5430} that explicitly model singular vectors of $\Ltn$ and argue about the infinity norm error using a Taylor series expansion. However, in our case, such an error analysis requires analyzing the following key quantities ($H=E_1+E_3$):
\begin{equation*}
      \forall 1\leq j,\  s.t.,\ j \text{ even}:
      \begin{aligned}
        A_j & :=\max\limits_{q \in [n]} \|e_{q}^\top\left(H^\top H\right)^\frac{j}{2}V^*\|_2 \\
        B_j & :=\max\limits_{q \in [m]} \|e_{q}^\top\left(HH^\top\right)^\frac{j}{2}U^*\|_2 ,\notag\\
      \end{aligned} \qquad \forall 1\leq j,\  s.t.,\ j \text{ odd}:
      \begin{aligned}
        C_j &:=  \max\limits_{q \in [n]} \|e_{q}^\top H^\top\left(HH^\top\right)^{\lfloor \frac{j}{2} \rfloor}U^*\|_2 \\
        D_j &:=\max\limits_{q \in [m]} \|e_{q}^\top H\left(H^\top H\right)^{\lfloor \frac{j}{2} \rfloor}V^*\|_2.\label{eq:keyQuants}
      \end{aligned}
\end{equation*}
Note that $E_1=0$ in the case of standard RPCA which was analyzed in \cite{NIPS2014_5430}, while $E_3=0$ in the case of standard MC which was considered in \cite{DBLP:conf/colt/0002N15}. In contrast, in our case both $E_1$ and $E_3$ are non-zero. Moreover, $E_3$ is dependent on random variable $\Omega$. Hence, for $j\geq 2$, we will get cross terms between $E_3$ and $E_1$ that will also have dependent random variables which precludes application of standard Bernstein-style tail bounds. To this end, we use a technique similar to that of \cite{ErdosKYY2012, DBLP:conf/colt/0002N15} to provide a careful combinatorial-style argument to bound the above given quantity. That is, we can provide the following key lemma:
\begin{lemma}
  \label{lem:einf}
  Let $\Lo$, $\Omega$, and $\So$ satisfy Assumptions 1, 2 and 3 respectively. Let $\Lo = U^{*} \Sigma^{*} (V^{*})^{\top}$ be the singular value decomposition of $\Lo$. Furthermore, suppose that in the $t^{\textrm{th}}$ iteration of the $q^{\textrm{th}}$ stage, $\Sot$ defined as $HT_{\zeta}(M - L^{(t)})$ satisfies $Supp(\Sot) \subseteq Supp(\So)$, then we have:
  \begin{equation*}
    \max \{A_a, B_a, C_a, D_a\}\leq \mu\sqrt{\frac{r}{m}}\left(\rho n \infnorm{E_1} \vphantom{\sqrt{\frac{n}{p}}} + c\sqrt{\frac{n}{p}} (\infnorm{E_{1}} + \infnorm{E_{2}}) \log n\right)^{2a + 1},
  \end{equation*}
  $\forall c > 0$ w.p $\geq 1 - n^{-2\log \frac{c}{4} + 4}$, where $E_{1}, E_{2} \text{ and } E_{3}$ are defined in (\ref{eqn:rUpdate}), $A_a, B_a, C_a, D_a$ are defined in \eqref{eq:keyQuants}.
\end{lemma}
{\em Remark}: We would like to note that even for the standard MC setting, i.e., when $E_1=0$, we obtain better bound than that of \cite{DBLP:conf/colt/0002N15} as we can bound $\max_i \|e_i^T(E_3)^qU\|_2$ directly rather than the weaker $\sqrt{r}\max_i \|e_i^T (E_3)^q u_j\|$ bound that \cite{DBLP:conf/colt/0002N15} uses.

Now, using Lemmas~\ref{lem:einf} and \ref{lem:hBound} and by using a hard-thresholding argument we can bound {$\|L^{(t + 1)}- \Lo\|_\infty\leq \frac{2\mu^2 r}{m} (\sigma^*_{k_{q}+1}+\left(\frac{1}{2}\right)^t\sigma^*_{k_{q}})$} (see Lemma~\ref{lem:Lprog}) in the $q$-th stage. Hence, after $O(\log(\sigma^*_1/\epsilon))$ ``inner'' iterations, we can guarantee in the $q$-th stage:
\begin{equation}
  \label{eq:innerL}
  \|L^{(T)}-\Lo\|_\infty \leq \frac{4\mu^2 r}{m} \sigma^*_{k_q+1} \quad \|E_1\|_\infty+\|E_2\|_\infty \leq \frac{20\mu^2 r}{m} \sigma^*_{k_q+1}.
\end{equation}

Moreover, by using sparsity of $\So$ and the special structure of $E_{3}$ (See Lemma~\ref{lem:hBound}), we have: ${\|E_1+E_3\|_2}\leq c \cdot \sigma^*_{k_q+1}$, where $c$ is a small constant.

Now, the outer iteration sets the next stage's rank $k_{q+1}$ as: $k_{q+1}=|\{i: \sigma_i(\Lo+E_1+E_3)\geq 0.5\cdot \sigma_{k_q+1}(\Lo+E_1+E_3)\}|$. Hence, using bound on $\|E_1+E_3\|_2$ and Weyl's eigenvalue perturbation bound (Lemma \ref{lem:weyl-perturbation}), we have: $\sigma^*_{k_{q+1}}\geq 0.6\, \sigma^*_{k_q+1}$ and $\sigma^*_{k_q+1}\leq \sigma^*_{k_q}$. Hence, after $Q=O(\log(\sigma_1^*/\epsilon))$ ``outer'' iterations, Algorithm~\ref{alg:gca} converges to an $\epsilon$-approximate solution to $\Lo$.

\section{Experiments}
\label{sec:exp}

In this section we discuss the performance of Algorithm \ref{alg:gca} on synthetic data and its use in foreground background separation. The goal of the section is two-fold: a) to demonstrate practicality and effectiveness of Algorithm \ref{alg:gca} for the RMC problem, b) to show that Algorithm~\ref{alg:gca} indeed solves RPCA problem in significantly smaller time than that required by the existing state-of-the-art algorithm (St-NcRPCA \cite{NIPS2014_5430}). To this end, we use synthetic data as well as video datasets where the goal is to perform foreground-background separation \cite{CandesLMW11}. 

We implemented our algorithm in MATLAB and the results for the synthetic data set were obtained by averaging over 20 runs. We obtained a matlab implementation of St-NcRPCA \cite{NIPS2014_5430} from the authors of \cite{NIPS2014_5430}. Note that if the sampling probability is $p=1$, then our method is similar to St-NcRPCA; the key difference being how rank is selected in each stage. We also implemented the Alternating Minimzation based algorithm from \cite{Wangaistats}. However, we found it to be an order of magnitude slower than Algorithm~\ref{alg:gca} on the foreground-background separation task. For example, on the escalator video, the algorithm did not converge in less than 150 seconds despite discounting for the expensive sorting operation in the truncation step. On the other hand, our algorithm finds the foreground in about 8 seconds.

\textit{Parameters.} The algorithm has three main parameters: 1) threshold $\eta$, 2) incoherence $\mu$ and 3) sampling probability $p$ ($E[|\Omega|]=p\cdot mn$). In the experiments on synthetic data we observed that keeping $\lambda\sim\mu\twonorm{M-\St}/{\sqrt{n}}$ speeds up the recovery while for background extraction keeping $\lambda\sim\mu\twonorm{M-\St}/{n}$ gives a better quality output. The value of $\mu$ for real world data sets was figured out using cross validation while for the synthetic data the same value was used as used in data generation. The sampling probability for the synthetic data could be kept as low as $2r\log^2(n)/{n}$ while for the real world data set we got good results for $p=0.05$. Also, rather than splitting samples, we use entire set of observed entries to perform our updates (see Algorithm~\ref{alg:gca}).

\begin{figure}
  \centering
  \subfloat[][]{
  \includegraphics[width=0.20\textwidth]{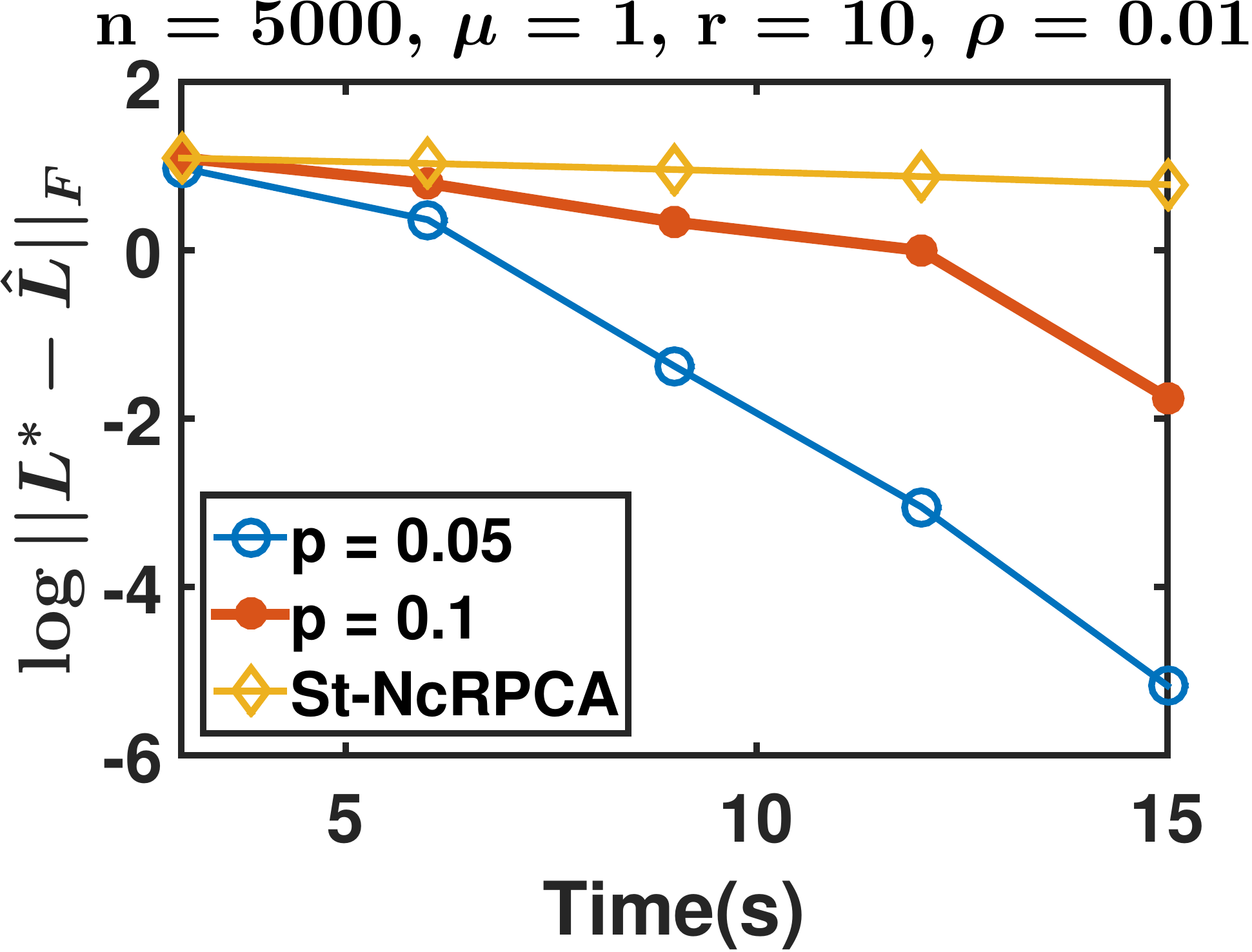}
  \label{fig:tFL}
  }
  \qquad
  \subfloat[][]{
  \includegraphics[width=0.20\textwidth]{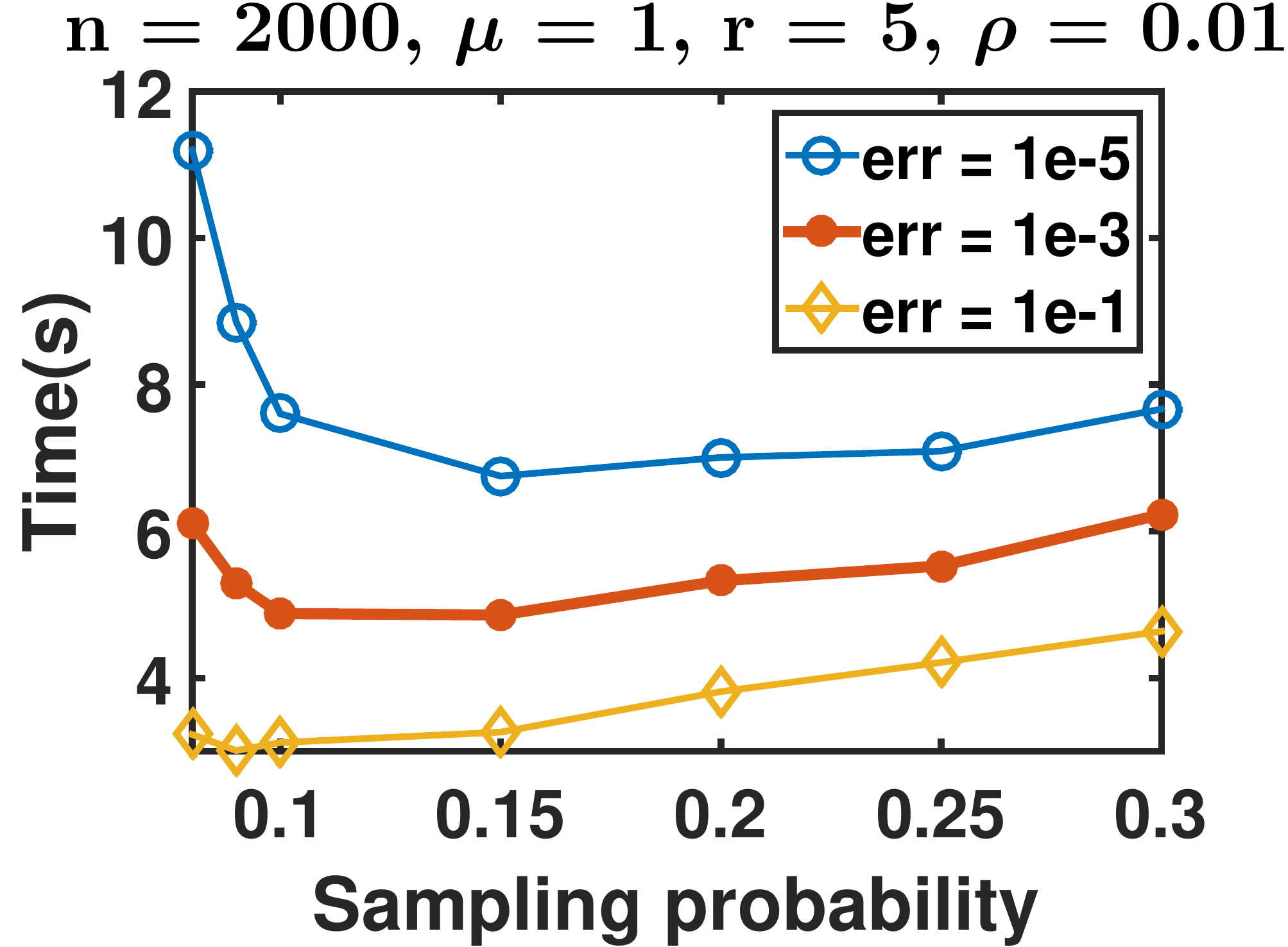}
  \label{fig:pVT}
  }
  \qquad
  \subfloat[][]{
  \includegraphics[width=0.20\textwidth]{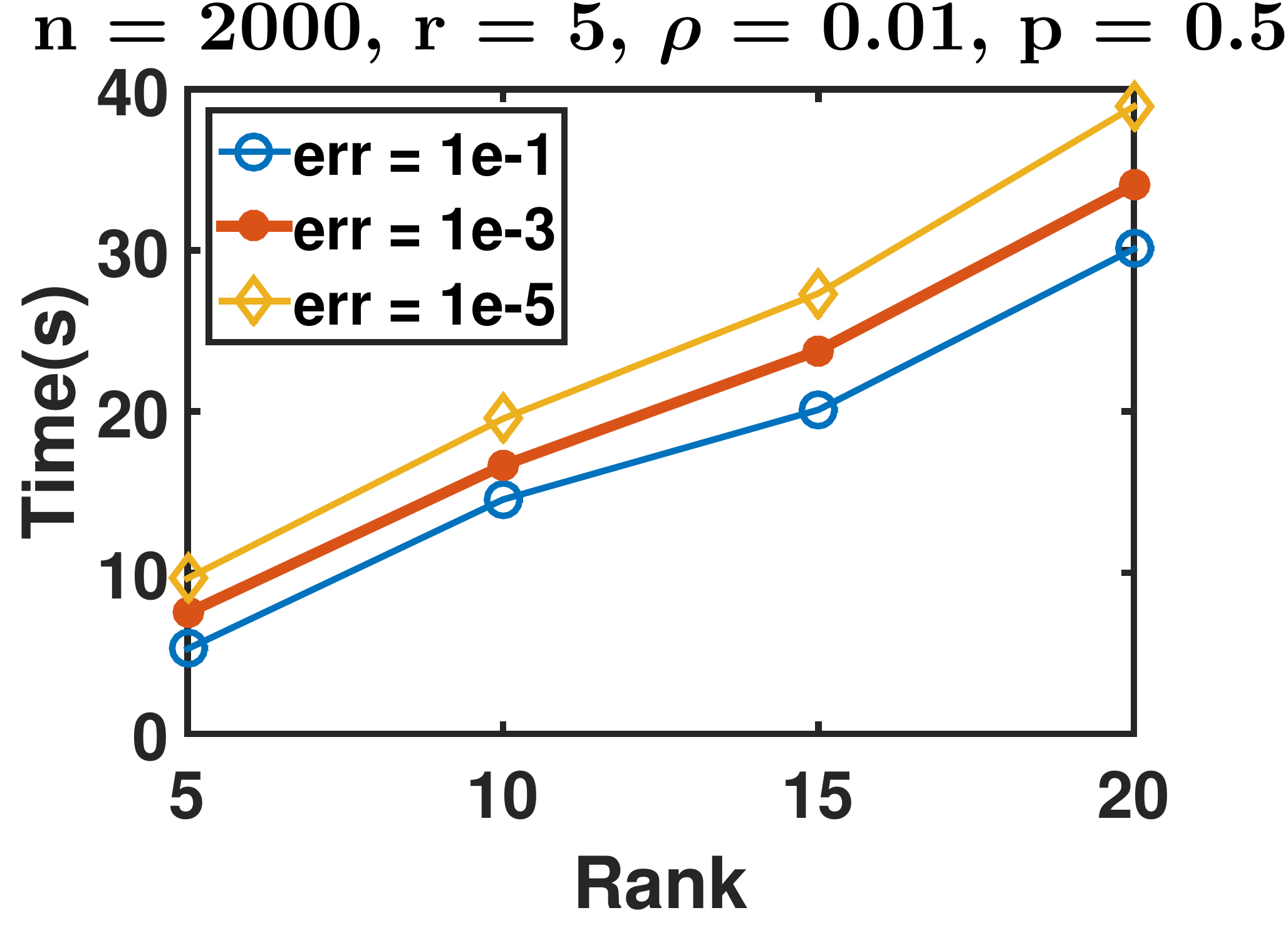}
  \label{fig:rVT}
  }
  \qquad
  \subfloat[][]{
  \includegraphics[width=0.20\textwidth]{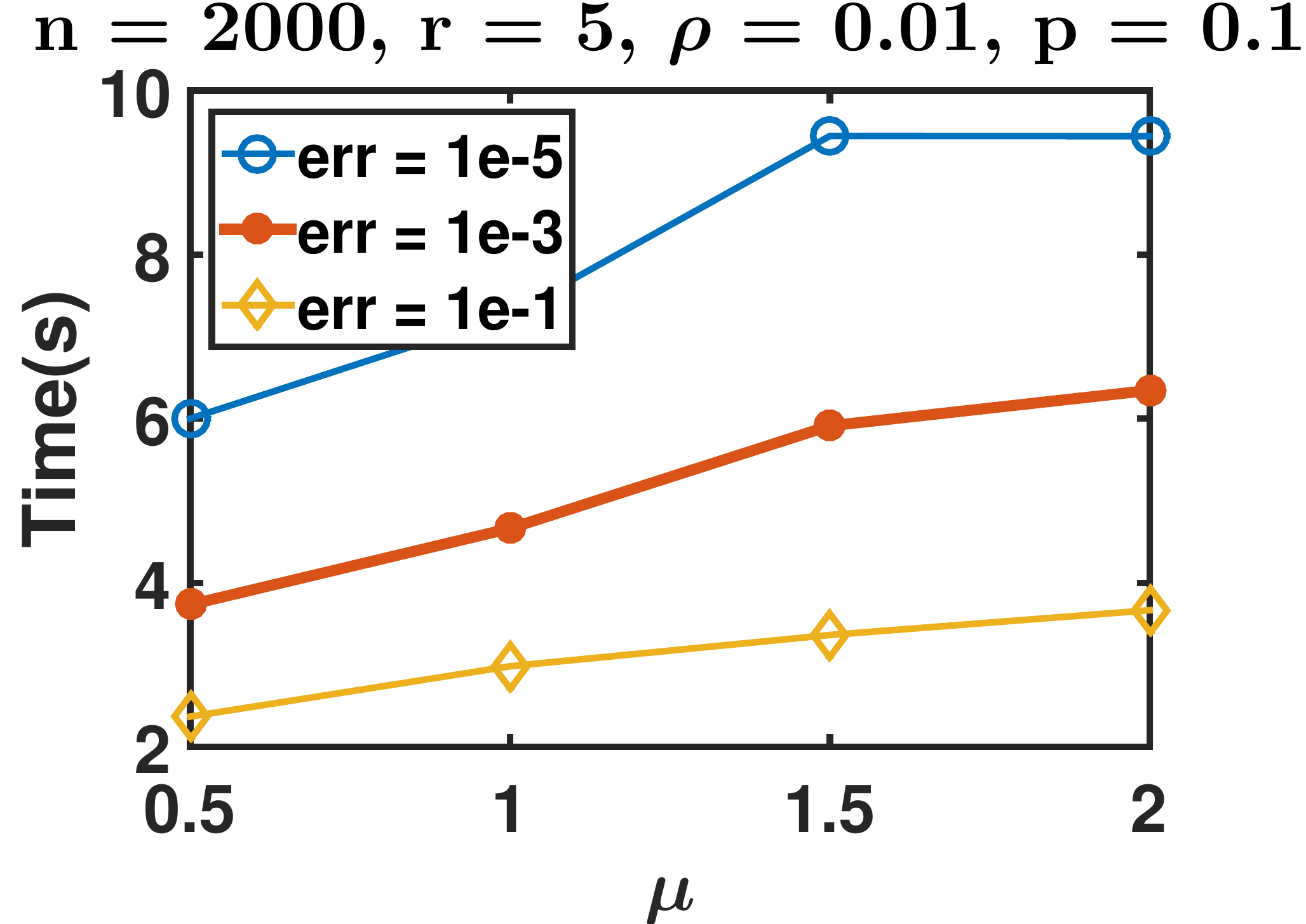}
  \label{fig:iVT}
  }
\vspace{.3in}
\caption{Performance of \ncgca on synthetic data. \ref{fig:tFL}: time vs error for various sampling probabilities; time taken by St-NcRMC \ref{fig:pVT}: sampling probability vs time for constant error; time taken decreases with decreasing sampling probability upto an extent and then increases \ref{fig:rVT}: time vs rank for constant error \ref{fig:iVT}: incoherence vs time for constant error}
\vspace{.3in}
\label{fig:synthetic_plots_1}
\end{figure}

\begin{figure}[t]
\centering
\subfloat[][]{
\includegraphics[width=0.20\textwidth]{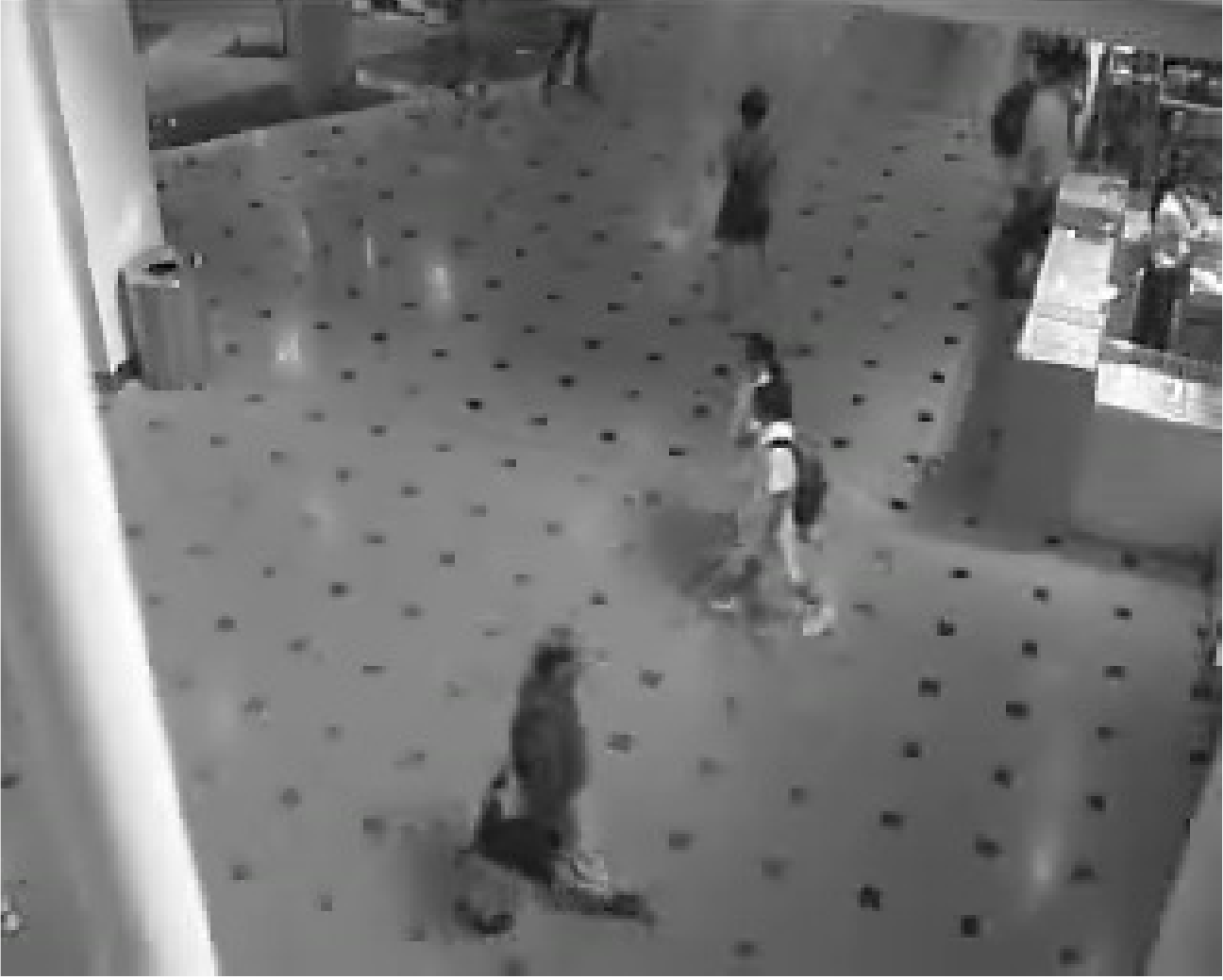}
\label{fig:shopO}
}
\qquad
\subfloat[][]{
\includegraphics[width=0.20\textwidth]{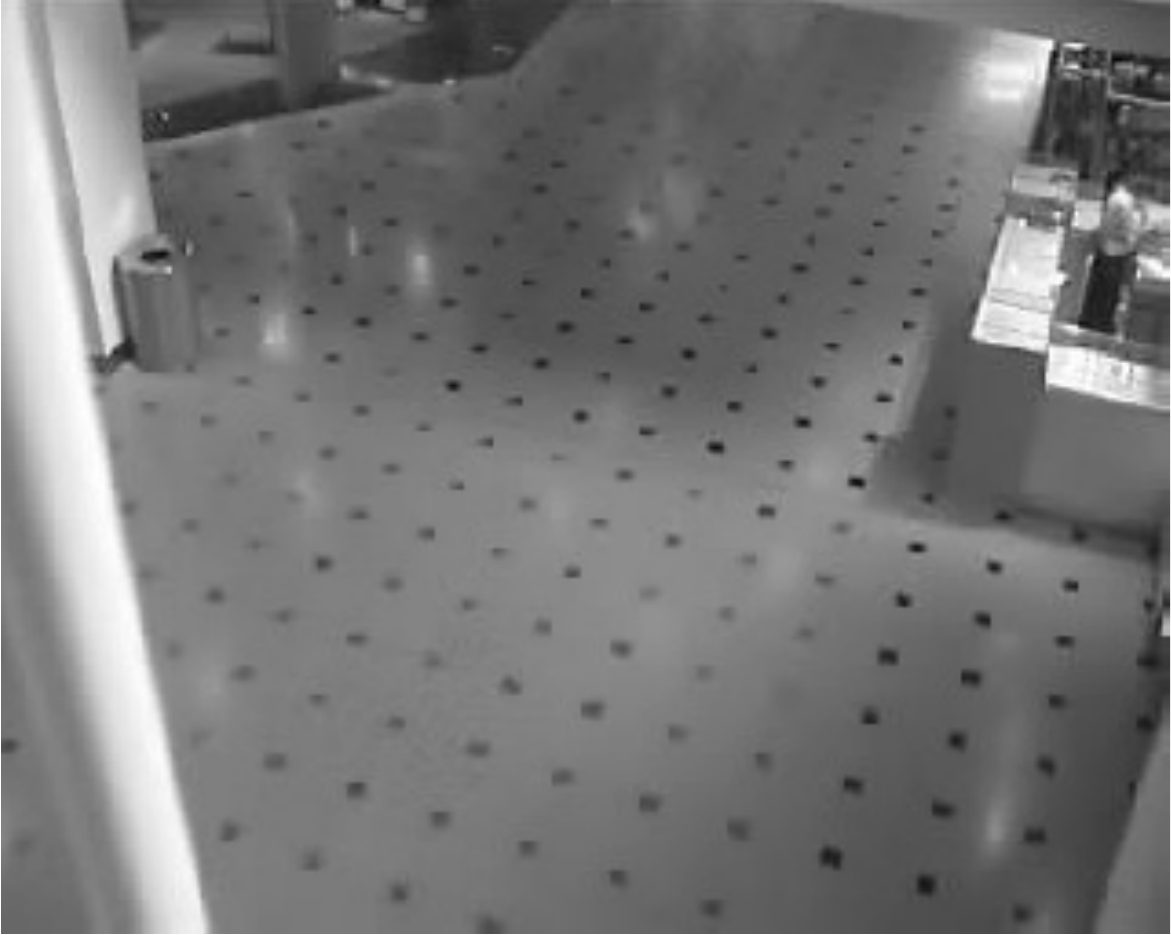}
\label{fig:shop}
}
\qquad
\subfloat[][]{
\includegraphics[width=0.20\textwidth, height=.18\textwidth]{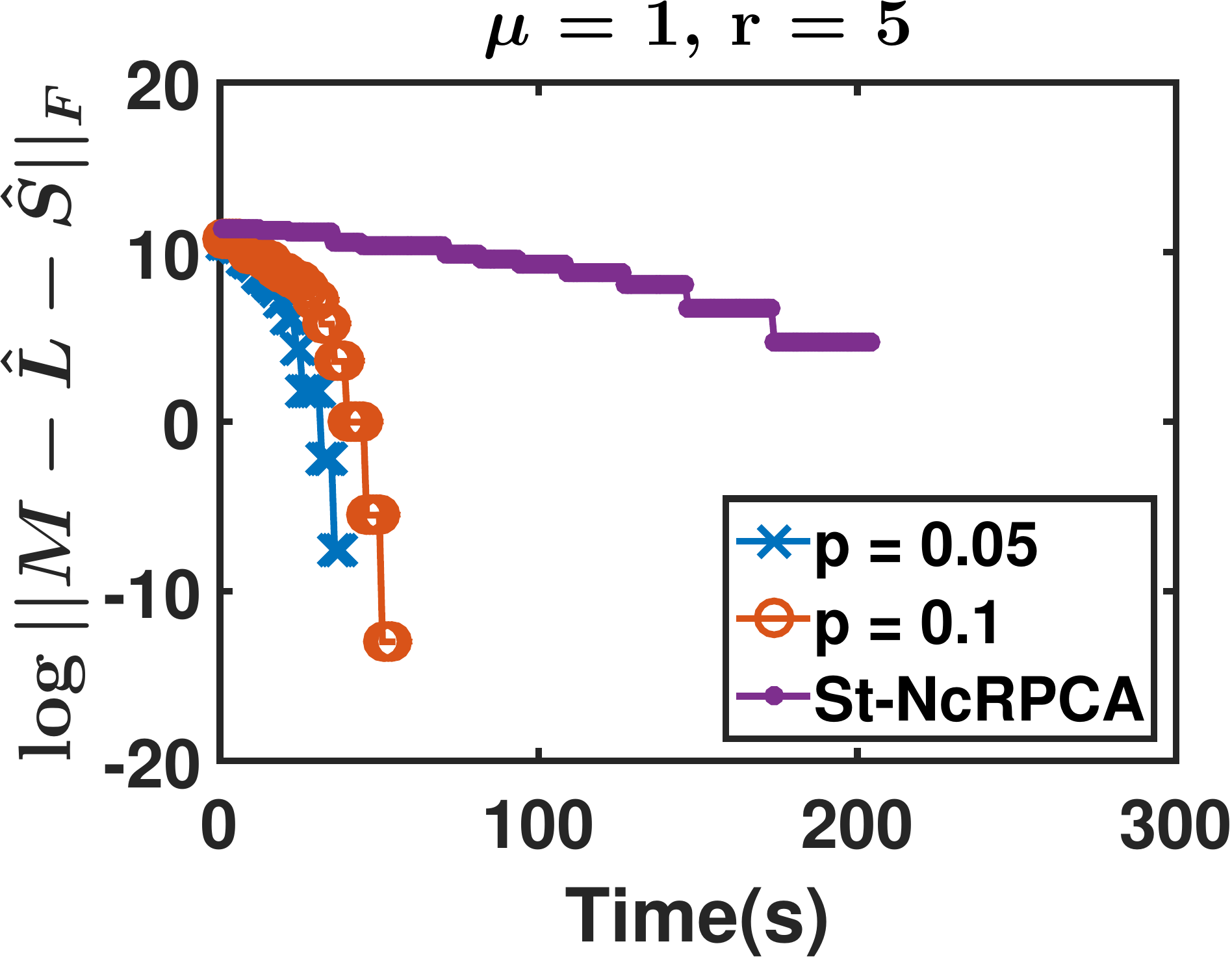}
\label{fig:shopFro}
}

\subfloat[][]{
\includegraphics[width=0.20\textwidth]{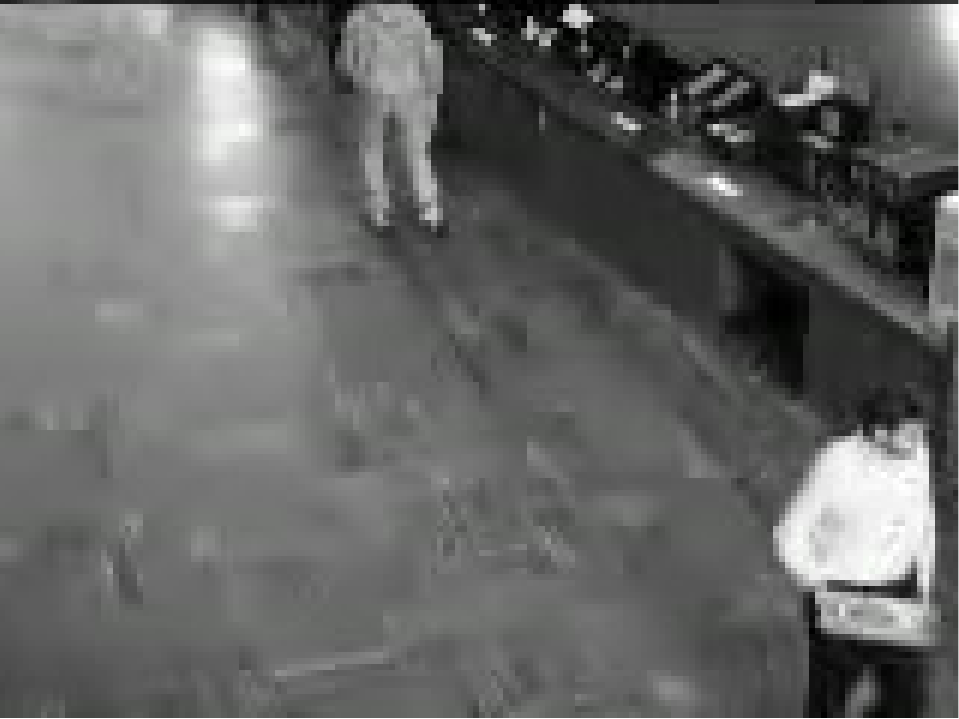}
\label{fig:resO}
}
\qquad
\subfloat[][]{
\includegraphics[width=0.20\textwidth]{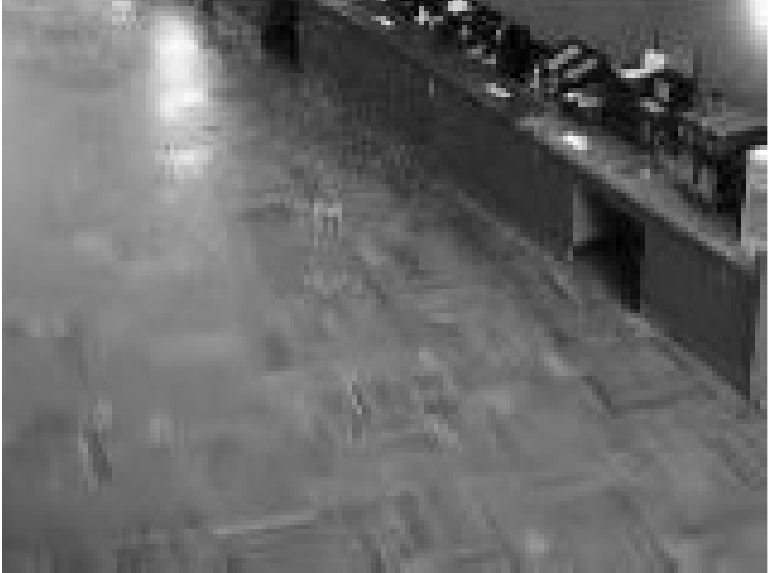}
\label{fig:res}
}
\qquad
\subfloat[][]{
\includegraphics[width=0.20\textwidth, height=.18\textwidth]{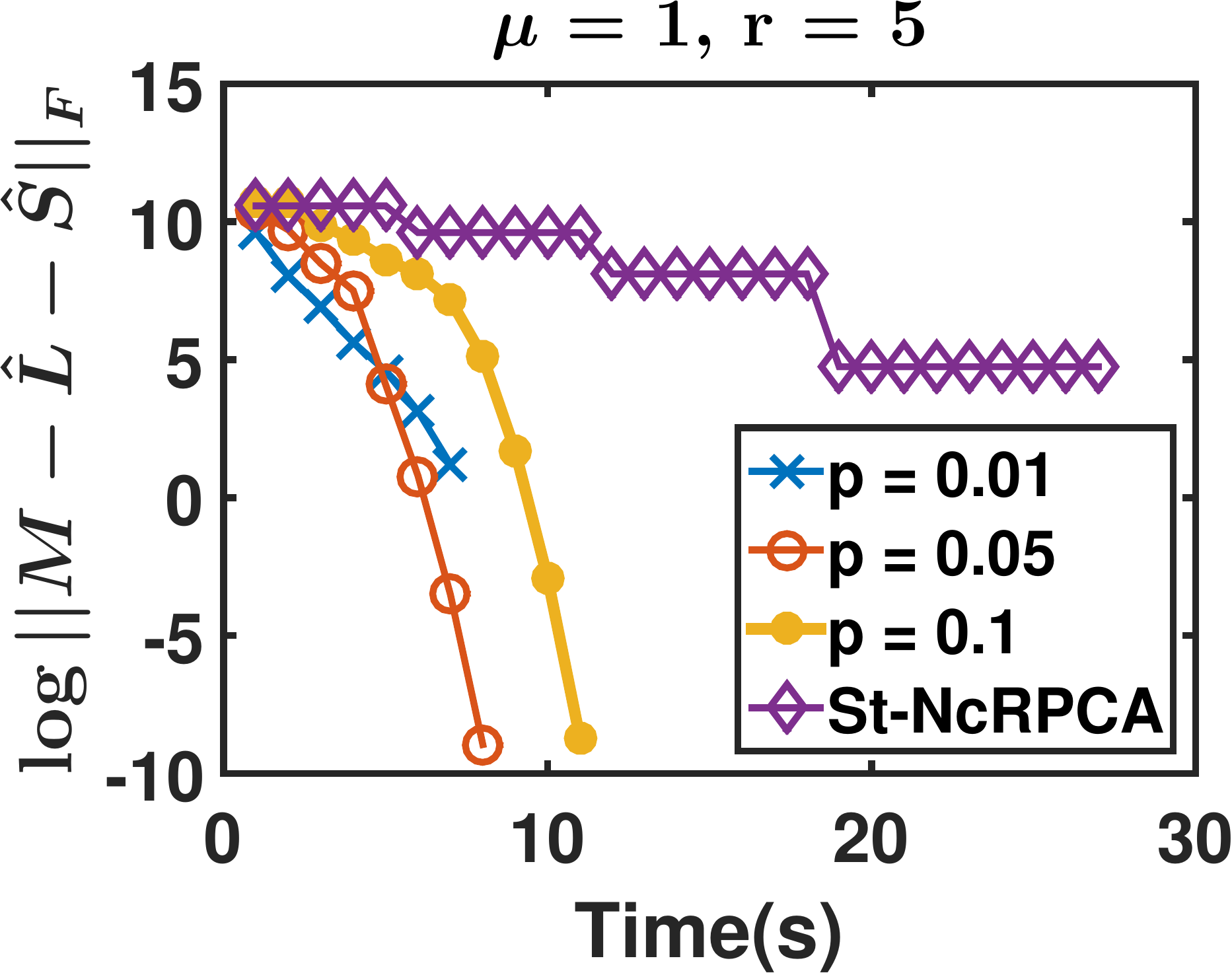}
\label{fig:resFro}
}
\vspace{.3in}
\caption{\ncgca on Shopping video. \ref{fig:shopO}: a video frame \ref{fig:shop}: an extracted background frame \ref{fig:shopFro}: time vs error for different sampling probabilities; \ncgca takes 38.7s while St-NcPCA takes 204.4s. \ncgca on Restaurant video. \ref{fig:resO}: a video frame \ref{fig:res}: an extracted background frame \ref{fig:resFro}: time vs error for different sampling probabilities; \ncgca takes 7.9s while St-NcPCA takes 27.8s}
\vspace{.3in}
\label{fig:real_plots_1}
\end{figure}


\textbf{Synthetic data.} We  generate $M=\Lo+\wSo$ of two sizes, where $\Lo=UV^{\top}\in \R^{2000\times2000}$ (and $\R^{5000\times5000}$) is a random rank-5 (and rank-10 respectively) matrix with incoherence $\approx 1$. $\wSo$ is generated by considering a uniformly random subset of size $\zeronorm{\So}$ from $[m]\times[n]$ where every entry is i.i.d. from the uniform distribution in $[\frac{r}{2\sqrt{mn}}, \frac{r}{\sqrt{mn}}]$. This is the same setup as used in \cite{CandesLMW11}.

Figure~\ref{fig:synthetic_plots_1} (a) plots recovery error ($\|L-\Lo\|_F$) vs computational time for our \ncgca method (with different sampling probabilities) as well as the St-NcRPCA algorithm. Note that even for very small values of sampling $p$, we can achieve same recovery error using significantly small values. For example, our method with $p=0.1$ achieve $0.01$ error ($\|L-\Lo\|_F$) in $\approx 2.5s$ while St-NcRPCA method requires $\approx 10s$ to achieve the same accuracy. Note that we do not compare against the convex relaxation based methods like IALM from \cite{CandesLMW11}, as \cite{NIPS2014_5430} shows that St-NcRPCA is significantly faster than IALM and several other convex relaxation solvers.

Figure~\ref{fig:synthetic_plots_1} (b) plots time required to achieve different recovery errors ($\|L-\Lo\|_F$) as the sampling probability $p$ increases. As expected, we observe a linear increase in the run-time with $p$. Interestingly, for very small values of $p$, we observe an increase in running time. In this regime, $\frac{\|\Pom(M)\|_2}{p}$ becomes very large (as $p$ doesn't satisfy the sampling requirements). Hence, increase in the number of iterations ($T\approx \log \frac{\|\Pom(M)\|_2}{p\epsilon}$) dominates the decrease in per iteration time complexity.

Figure~\ref{fig:synthetic_plots_1} (c), (d) plots computation time required by our method (\ncgca, Algorithm~\ref{alg:gca}) versus rank and incoherence, respectively. As expected, as these two problem parameters increase, our method requires more time. Note that our run-time dependence on rank seems to be linear, while our existing results require $O(r^3)$ time. This hints at the possibility of further improving the computational complexity analysis of our algorithm.

We also study phase transition for different values of sampling probability $p$. Figure~\ref{fig:synth_plots_1_app} (a) in Appendix~\ref{app:exp} show a phase transition phenomenon where beyond $p>.06$ the probability of recovery is almost $1$ while below it, it is almost $0$.


\textbf{Foreground-background separation.} We also applied our technique to the problem of foreground-background separation. We use the usual method of stacking up the vectorized video frames to construct a matrix. The background, being static, will form the low rank component while the foreground can be considered to be the noise.

We applied our \ncgca method (with varying $p$) to several videos. Figure~\ref{fig:real_plots_1} (a), (d) shows one frame each from two videos (a shopping center video, a restaurant video). Figure~\ref{fig:real_plots_1} (b), (d) shows the extracted background from the two videos by using our method (\ncgca, Algorithm~\ref{alg:gca}) with probability of sampling $p=0.05$. Figure~\ref{fig:real_plots_1} (c), (f) compares objective function value for different $p$ values. Clearly, \ncgca can recover the true background with $p$ as small as $0.05$. We also observe an order of magnitude speedup ($\approx 5$x) over St-NcRPCA \cite{NIPS2014_5430}. We present results on the video Escalator in Appendix \ref{app:exp}.





\textbf{Conclusion.} In this work, we studied the Robust Matrix Completion problem. For this problem, we provide exact recovery of the low-rank matrix $\Lo$ using nearly optimal number of observations as well as nearly optimal fraction of corruptions in the observed entries. Our RMC result is based on a simple and efficient PGD algorithm that has nearly linear time complexity as well. Our result improves state-of-the-art for the related Matrix Completion as well as Robust PCA problem. For Robust PCA, we provide first nearly linear time algorithm under standard assumptions.

Our sample complexity depends on $\epsilon$, the desired accuracy in $\Lo$. Moreover, improving dependence of sample complexity on $r$ (from $r^2$ to $r$) also represents an important direction. Finally, similar to foreground background separation, we would like to explore more applications of RMC/RPCA.

{\small \bibliography{refs}}
\bibliographystyle{alpha}
 \clearpage
\section{Appendix}\label{sec:appendix}

We divide this section into five parts. In the first part we prove some common lemmas. In the second part we give the convergence guarantee for \ncgca. In the third part we give another algorithm which has a sample complexity of $O(\mu^4r^3n\log^2{n}\log{\frac{\mu^2r\sigma_1^*}{\epsilon}})$ and prove its convergence guarantees. In the fourth part we prove a generalized form of lemma \ref{lem:einf}. In the fifth part we present some additional experiments.


For the sake of convenience in the following proofs, we will define some notations here.

We define $p=\frac{\abs{\Omega_{k, t}}}{mn}$ and we consider the following equivalent update step for $\Ltn$ in the analysis:
\begin{center}
  \begin{tabular}{ c c }
    $\Ltn \coloneqq \Pk(M^{(t)})$ & $\Mt \coloneqq \Lo+H$ \\
    $H \coloneqq \Et + \beta G$ & $\Et \coloneqq \So-\Sot$ \\
    $\Sot \coloneqq \HT\left(\Mt-\Lt\right)$ & $G \coloneqq \ddfrac{1}{\beta}\left(\mathcal{I}-\ddfrac{\Pomqt}{p}\right)D$ \\
    $D \coloneqq \Lt-\Lo+\Sot-\So$ & $\beta \coloneqq \ddfrac{2\sqrt{n}\infnorm{D}}{\sqrt{p}}$
  \end{tabular}
\end{center}


The singular values of $L^*$ are denoted by $\sigma^*_1,\ldots,\sigma^*_r$ where $\abs{\sigma^*_1}\geq\ldots\geq\abs{\sigma^*_r}$ and we will let $\lambda_{1},\ldots,\lambda_{n}$ denote the singular values of $\Mt$ where $\abs{\lambda_{1}} \geq \ldots \geq \abs{\lambda_{n}}$.

\subsection{Common Lemmas}
We will begin by restating some lemmas from previous work that we will use in our proofs.

First, we restate Weyl's perturbation lemma from \cite{bhatia}, a key tool in our analysis:
\begin{lemma}\label{lem:weyl-perturbation}
  Suppose $B = A + E \in \mathbb{R}^{m \times n}$ matrix. Let $\lambda_1,\cdots,\lambda_k$ and $\sigma_1,\cdots,\sigma_k$ be the singular values of $B$ and $A$ respectively such that $\lambda_1 \geq \cdots \geq \lambda_k$ and $\sigma_1 \geq \cdots \geq {\sigma_k}$. Then:
  \begin{align*}
    \abs{\lambda_i - \sigma_i} \leq \twonorm{E} \; \forall \; i \in [k].
  \end{align*}
\end{lemma}

This lemma establishes a bound on the spectral norm of a sparse matrix.
\begin{lemma}
  \label{lem:spSpec}
  Let $S \in \mathbb{R}^{m\times n}$ be a sparse matrix with row and column sparsity $\rho$. Then,
  $$\twonorm{S} \leq \rho \max \{m, n\} \infnorm{S}$$
\end{lemma}

\begin{proof}
  For {\em any} pair of unit vectors $u$ and $v$, we have:
  \begin{align*}
    v^{\top}S u &= \sum\limits_{1 \leq i \leq m, 1 \leq j \leq n}  v_{i}u_{j}S_{ij} \leq \sum\limits_{1 \leq i \leq m, 1 \leq j \leq n} \abs{S_{ij}} \left(\frac{v_{i}^{2} + u_{j}^{2}}{2}\right) \\&\leq \frac{1}{2}\left(\sum\limits_{1 \leq i \leq m} v_{i}^{2} \sum\limits_{1 \leq j \leq n}\abs{S_{ij}} + \sum\limits_{1 \leq j \leq n} u_{j}^{2} \sum\limits_{1 \leq i \leq m}\abs{S_{ij}}\right) \leq \rho \max\{m, n\} \infnorm{S}
  \end{align*}
  Lemma now follows by using $\|S\|_2=\max_{u, v, \|u\|_2=1, \|v\|_2=1}u^T S v$.
\end{proof}

Now, we define a $0$-mean random matrix with small higher moments values. 
\begin{definition}[Definition 7, \cite{DBLP:conf/colt/0002N15}]
  \label{def:mzero}
  $H$ is a random matrix of size $m \times n$ with each of its entries drawn independently satisfying the following moment conditions:
  \begin{center}
    \begin{tabular}{ccc}
      $\mathbb{E}[h_{ij}] = 0,$ & $\abs{h_{ij}} < 1,$ &$\mathbb{E}[\abs{h_{ij}}^k] \leq \frac{1}{\max\{m, n\}}$,
    \end{tabular}
  \end{center}
  for $i,j \in [n]$ and $2 \leq k \leq 2 \log n$.
\end{definition}

We now restate two useful lemmas from \cite{DBLP:conf/colt/0002N15}:
\begin{lemma}[Lemma 12, 13 of \cite{DBLP:conf/colt/0002N15}]
  We have the following two claims:
  \begin{itemize}
    \item Suppose $H$ satisfies Definition \ref{def:mzero}.
    Then, w.p. $\geq 1-1/n^{10+\log \alpha}$, we have: $\twonorm{H}\leq 3\sqrt{\alpha}.$
    \item Let $A$ be a $m \times n$ matrix with $n \geq m$. Suppose $\Omega \subseteq [m]\times[n]$ is obtained by sampling each element
    with probability $p \in \left[\frac{1}{4n},0.5\right]$. Then, the following  matrix $H$ satisfies Defintion~\ref{def:mzero}:
    \[
    H \coloneqq \frac{\sqrt{p}}{2 \sqrt{n} \infnorm{\A}} \left(A - \frac{1}{p}\Pom(A)\right).
    \]
  \end{itemize}
  \label{lem:satDefn}
  \label{lem:defSpec}
\end{lemma}




\begin{lemma}[Lemma 13, \cite{DBLP:conf/colt/0002N15}]
  \label{lem:geomBound}
  Let $A \in \mathbb{R}^{n \times n}$ be a symmetric matrix with eigenvalues $\sigma_{1},\cdots,\sigma_{n}$ where $\abs{\sigma_{1}} \geq \cdots \geq \abs{\sigma_{n}}$. Let $B = A + C$ be a perturbation of $A$ satisfying $\twonorm{C} \leq \frac{\sigma_{k}}{2}$ and let $\Pk(B) = U\Lambda U^{\top}$ by the rank-$k$ projection of B. Then, $\Lambda^{-1}$ exists and we have:
  \begin{enumerate}
    \item $\twonorm{A - AU\Lambda^{-1}U^{\top}A} \leq \abs{\sigma_{k}} + 5\twonorm{C}$,
    \item $\twonorm{AU\Lambda^{-a}U^{\top}A} \leq 4 \left( \frac{\abs{\sigma_{k}}}{2} \right)^{-a + 2} \quad \forall a \geq 2$.
  \end{enumerate}
\end{lemma}

We now provide a lemma that bounds $\|\cdot\|_\infty$ norm of an incoherent matrix with its operator norm.
\begin{lemma}
  \label{lem:incohBound}
  Let $A \in \mathbb{R}^{m \times n}$ be a rank $r$, $\mu$-incoherent matrix. Then for any  $C \in \mathbb{R}^{n \times m}$, we have:
  \begin{equation*}
    \infnorm{ACA} \leq \frac{\mu^{2}r}{\sqrt{mn}} \twonorm{ACA}
  \end{equation*}
\end{lemma}
\begin{proof}
  Let $A=U\Sigma V^\top$. Then, $ACA=UU^\top ACAVV^\top$. The lemma now follows by using definition of incoherence with the fact that $\|U^\top ACAV\|_2\leq \|ACA\|_2$.
\end{proof}

We now present a lemma that shows improvement in the error $\|L-\Lo\|_\infty$ by using gradient descent on $\Lt$.
\begin{lemma}
  \label{lem:hBound}
  Let $\Lo$, $\Omega$, $\wSo$ satisfy Assumptions 1,2,3 respectively. Also, let the following hold for the $t$-th inner-iteration of any stage $q$:
  \begin{enumerate}
    \item $\infnorm{\Lo - \Lt} \leq  \frac{2\mu^2r}{m} \left(\sigma^{*}_{k + 1} + \left(\frac{1}{2}\right)^{z}\sigma^{*}_{k}\right)$
    \item $\infnorm{\wSo - \Sot} \leq  \frac{8\mu^2r}{m} \left(\sigma^{*}_{k + 1} + \left(\frac{1}{2}\right)^{z}\sigma^{*}_{k}\right)$
    \item $Supp(\Sot) \subseteq Supp(\wSo)$
  \end{enumerate}
  where $z\geq-3$ and $\sigma_k^*$ and $\sigma_{k+1}^*$ are the $k$ and $(k+1)^{\textit{th}}$ singular values of $\Lo$. Also, let $E_1=\Sot - \wSo$ and $E_3=\left(\mathcal{I} - \frac{\Pomqt[t]}{p}\right) \left(\Lt-\Lo+\Sot-\wSo\right)$ be the error terms defined also in \eqref{eqn:rUpdate}. Then, the following holds w.p $\geq 1 - n^{-(10 + \log \alpha)}$:
  \begin{equation}
    \twonorm{E_1+E_3} \leq \frac{1}{100} \left(\sigma^{*}_{k+1} + \left(\frac{1}{2}\right)^{z} \sigma^{*}_{k}\right)
  \end{equation}
\end{lemma}

\begin{proof}
  Note from Lemma \ref{lem:satDefn}, $$\frac{1}{\beta}E_3=\frac{1}{\beta} \left(\mathcal{I} - \frac{\Pomqt[t]}{p}\right) \left(\Lt-\Lo+\Sot-\wSo\right),$$ satisfies definition \ref{def:mzero} with $\beta=\frac{2\sqrt{n}}{\sqrt{p}}\cdot \|\Lt-\Lo+\wSt-\wSo\|_\infty$.

  We now bound the spectral norm of $E_1+E_3$ as follows:
  \begin{align*}
    \twonorm{E_1+E_3} &\leq \twonorm{E_1}+\beta\cdot \twonorm{\frac{1}{\beta}E_3}
    \overset{(\zeta_{1})}{\leq} \rho n \infnorm{\Sot - \wSo} + 3\beta\sqrt{\alpha},\\
    &\overset{(\zeta_{2})}{\leq} \frac{1}{200} \left(\sigma^{*}_{k_{q} + 1} + \left(\frac{1}{2}\right)^{z} \sigma^{*}_{k_{q}}\right) + \frac{60\mu^{2}r}{m} \sqrt{\frac{n}{p}} \sqrt{\alpha} \left(\abs{\sigma^{*}_{k_{q} + 1}} + \left(\frac{1}{2}\right)^{z} \abs{\sigma^{*}_{k_{q}}}\right),\\
    &\overset{(\zeta_3)}{\leq} \frac{1}{100} \left(\sigma^{*}_{k_{q} + 1} + \left(\frac{1}{2}\right)^{z} \sigma^{*}_{k_{q}}\right).
  \end{align*}
  where $(\zeta_{1})$ follows from Lemma \ref{lem:spSpec} and \ref{lem:defSpec}. $(\zeta_{2})$ follows by our assumptions on $\rho$, $\infnorm{\Lt-\Lo}$, and $\infnorm{\Sot - \wSo}$. $(\zeta_{3})$ follows from our assumption on $p$. 
\end{proof}

In the following lemma, we prove that the value of the threshold computed using $\sigma_{k}(M^{(t)})=\sigma_k(\Lo+E_1+E_3)$, where $E_1, E_3$ are defined in \eqref{eqn:rUpdate}, closely tracks the threshold that we would have gotten had we had access to the true eigenvalues of $\Lo$, $\sigma^{*}_{k}$.

\begin{lemma}
  \label{lem:eigTrack}
  Let $\Lo$, $\Omega$, $\wSo$ satisfy Assumptions 1,2,3 respectively. Also, let the following hold for the $t$-th inner-iteration of any stage $q$:
  \begin{enumerate}
    \item $\infnorm{\Lo - \Lt} \leq  \frac{2\mu^2r}{m} \left(\sigma^{*}_{k + 1} + \left(\frac{1}{2}\right)^{z}\sigma^{*}_{k}\right)$
    \item $\infnorm{\wSo - \Sot} \leq  \frac{8\mu^2r}{m} \left(\sigma^{*}_{k + 1} + \left(\frac{1}{2}\right)^{z}\sigma^{*}_{k}\right)$
    \item $Supp(\Sot) \subseteq Supp(\wSo)$
  \end{enumerate}
  where $z\geq-3$ and $\sigma_k^*$ and $\sigma_{k+1}^*$ are the $k$ and $(k+1)^{\textit{th}}$ singular values of $\Lo$. Also, let $E_1=\Sot - \wSo$ and $E_3=\left(\mathcal{I} - \frac{\Pomqt[t]}{p}\right) \left(\Lt-\Lo+\Sot-\wSo\right)$ be the error terms defined also in \eqref{eqn:rUpdate}. Then, the following holds $\forall z > -3$ w.p $\geq 1 - n^{-(10 + \log \alpha)}$:
  \begin{equation}
    \ddfrac{7}{8} \left(\sigma^{*}_{k + 1} + \left(\ddfrac{1}{2}\right)^{z + 1}\sigma^{*}_{k}\right) \leq \left(\lambda_{k + 1} + \left(\ddfrac{1}{2}\right)^{z + 1}\lambda_{k}\right) \leq \ddfrac{9}{8} \left(\sigma^{*}_{k + 1} + \left(\ddfrac{1}{2}\right)^{z + 1}\sigma^{*}_{k}\right),
  \end{equation}
  where $\lambda_k:=\sigma_{k}(M^{(t)})=\sigma_k(\Lo+E_1+E_3)$ and $E_1, E_3$ are defined in \eqref{eqn:rUpdate}.
\end{lemma}

\begin{proof}
  Using Weyl's inequality (Lemma \ref{lem:weyl-perturbation}), we have: : $\abs{\lambda_{k} - \sigma^{*}_{k}}\leq \|E_1+E_3\|_2$ and $\abs{\lambda_{k + 1} - \sigma^{*}_{k + 1}}\leq \|E_1+E_3\|_2$
  We now proceed to prove the lemma as follows:
  \begin{align*}
    &\abs{\lambda_{k + 1} + \left(\ddfrac{1}{2}\right)^{z + 1}\lambda_{k} - \sigma^{*}_{k + 1} - \left(\ddfrac{1}{2}\right)^{z + 1}\sigma^{*}_{k}} \leq \abs{\lambda_{k + 1} - \sigma^{*}_{k + 1}} + \left(\ddfrac{1}{2}\right)^{z + 1} \abs{\lambda_{k} - \sigma^{*}_{k}}, \\
    &\hspace*{2cm}\leq \twonorm{E_1+E_3} \left(1 + \left( \frac{1}{2} \right)^{z + 1}\right) \overset{(\zeta)}{\leq} \frac{1}{100} \left(\sigma^{*}_{k + 1} + \left( \frac{1}{2} \right)^{z}\sigma^{*}_{k}\right) \left(1 + \left( \frac{1}{2} \right)^{z + 1}\right), \\
    &\hspace*{2cm}\leq \frac{1}{8} \left(\sigma^{*}_{k + 1} + \left( \frac{1}{2} \right)^{z + 1}\sigma^{*}_{k}\right),
  \end{align*}
  where $(\zeta)$ follows from Lemma \ref{lem:hBound} and the last inequality follows from the assumption that $z \geq -3$.
\end{proof}

Next, we show that the projected gradient descent update \eqref{eqn:rUpdate} leads to a better estimate of $\Lo$, i.e., we bound $\|\Ltn-\Lo\|_\infty$. Under the assumptions of the below given Lemma, the proof follows arguments similar to \cite{NIPS2014_5430} with additional challenge arises due to more involved error terms $E_1$, $E_3$.

Our proof proceeds by first symmetrizing our matrices by rectangular dilation. We first begin by noting some properties of symmetrized matrices used in the proof of the following lemma.

\begin{remark}
  \label{rem:sym}
  Let $A$ be a $m\times n$ dimensional matrix with singular value decomposition $U\Sigma V^{\top}$. We denote its symmetrized version be $A_s \coloneqq \begin{bmatrix} 0 & A^{\top} \\ A & 0 \end{bmatrix}$. Then:

  \begin{enumerate}
    \item The eigenvalue decomposition of $A_s$ is given by $A_s = U_s \Sigma_s U_s^{\top}$ where
    \[U_s \coloneqq \frac{1}{\sqrt{2}}\begin{bmatrix} V & V \\ U & -U \end{bmatrix} \hfill \Sigma_s \coloneqq \begin{bmatrix} \Sigma & 0 \\ 0 & -\Sigma \end{bmatrix}\]
    \item $\Pk[2k] \left( A_s \right) = \begin{bmatrix} 0 & \Pk(A^{\top}) \\ \Pk(A) & 0 \end{bmatrix}$
    \item We have
    \begin{tabular}{c c}
      $A_s^{2j}=\begin{bmatrix} (A^\top A)^j & 0 \\ 0 & (AA^\top)^j \end{bmatrix}$ & $A_s^{2j+1}=\begin{bmatrix} 0 & (A^\top A)^jA^\top \\ (AA^\top)^jA & 0 \end{bmatrix}$
    \end{tabular}
    \item We have
    \[U_s \Sigma_s^{-j} U_s^\top = \begin{bmatrix} V \Sigma^{-j} V^\top & 0 \\ 0 & U \Sigma^{-j} U^\top \end{bmatrix} \mbox{when $j$ is even}\]
    \[U_s \Sigma_s^{-j} U_s^\top = \begin{bmatrix} 0 & V \Sigma^{-j} U^\top \\ U \Sigma^{-j} V^\top & 0\end{bmatrix} \mbox{when $j$ is odd}\]
  \end{enumerate}
\end{remark}


\begin{lemma}
  \label{lem:Lprog}
  Let $\Lt=\proj_k(\Lo+H)$, where $H$ is any perturbation matrix that satisfies the following:
  \begin{enumerate}
    \item $\twonorm{H} \leq \ddfrac{\sigma^{*}_{k}}{4}$ \label{nrmBnd}
    \item $\forall i \in [n], \ a \in \lceil \frac{\log n}{2} \rceil$ with $\upsilon \leq \ddfrac{\sigma^{*}_{k}}{4}$
    \begin{center}
      \begin{tabular}{c c}
        $\twonorm{e_{i}^\top\left(H^\top H\right)^aV^*}, \twonorm{e_{i}^\top\left(HH^\top\right)^aU^*} \leq (\upsilon)^{a} \mu\sqrt{\ddfrac{r}{m}}$ \\
        $\twonorm{e_{i}^\top H^\top\left(HH^\top\right)^aU^*}, \twonorm{e_{i}^\top H\left(H^\top H\right)^aV^*} \leq (\upsilon)^{a} \mu\sqrt{\ddfrac{r}{m}}$
      \end{tabular}
    \end{center}
  \end{enumerate}
  where $\sigma_k^*$ is the $k^{\textit{th}}$ singular value of $L^*$. Also, let $\Lo$ satisfy Assumption 1. Then, the following holds:
  \begin{equation*}
    \infnorm{\Lt[t+1] - \Lo} \leq \frac{\mu^{2}r}{m} \left(\sigma^{*}_{k + 1} + 20\twonorm{H} + 8\upsilon\right)
  \end{equation*}
  where $\mu$ and $r$ are the rank and incoherence of the matrix $\Lo$ respectively.
\end{lemma}

\begin{proof}
  \begin{equation*}
    \Lts[t + 1] = \Pk[2k]\left(\Los + \Hs\right)
  \end{equation*}

  Let $l = m + n$. Let $\lambda_{1}, \cdots, \lambda_{l}$ be the eigenvalues of $\Mts=\Los+\Hs$ with $\abs{\lambda_{1}} \geq \abs{\lambda_{2}} \cdots \geq \abs{\lambda_{l}}$. Let $u_{1}, u_{2}, \cdots, u_{l}$ be the corresponding eigenvectors of $\Mts$. Using Lemma \ref{lem:weyl-perturbation} along with the assumption on $\twonorm{\Hs}$, we have: $|\lambda_{2k}| \geq \frac{3\sigma^{*}_{k}}{4}$.

  Let $U\Lambda V$ be the eigen vector decomposition of $\Ltn$.
  Let $U_s\Lambda_s U_s^{\top}$ to be the eigen vector decomposition of  $\Lts[t+1]$. Then, using Remark \ref{rem:sym} we have $\forall\ i \in [2k]$:
  \begin{equation*}
    \left(\Los + \Hs\right)u_{i} = \lambda_{i}u_{i},\ \text{ i.e. }\ \left(I - \frac{\Hs}{\lambda_{i}}\right)u_{i} = \Los u_{i}.
  \end{equation*}
  As $\abs{\lambda_{2k}} \geq \frac{3\sigma^{*}_{k}}{4}$ and $\twonorm{\Hs} \leq \frac{1}{4}\sigma_{k}^{*}$, we can apply the Taylor's series expansion to get the following expression for $u_{i}$:

  \begin{equation*}
    u_{i} = \frac{1}{\lambda_{i}} \left(I + \sum\limits_{j = 0}^{\infty} \left(\frac{\Hs}{\lambda_{i}}\right)^{j}\right)\Los u_{i}.
  \end{equation*}
  That is,
  \begin{align*}
    \Lts[t+1] &= \sum\limits_{i = 1}^{2k} \lambda_{i} u_{i}u_{i}^{\top} = \sum\limits_{i = 1}^{2k} \lambda_{i}^{-1} \sum_{0 \leq s, t < \infty} \left(\frac{\Hs}{\lambda_{i}}\right)^{s} \Los u_{i} u_{i}^{\top} \Los \left(\frac{\Hs}{\lambda_{i}}\right)^{t}, \\
    &= \sum_{0 \leq s, t < \infty} \sum\limits_{i = 1}^{2k} \lambda_{i}^{-(s + t + 1)} \Hs^{s} \Los u_{i} u_{i}^{\top} \Los \Hs^{t} = \sum_{0 \leq s, t < \infty} \Hs^{s} \Los U_s \Lambda_s^{-(s + t + 1)} U_s^{\top} \Los \Hs^{t}.
  \end{align*}
  Subtracting $\Los$ on both sides and taking operator norm, we get:
  \begin{align}
    \hspace*{-15pt}\infnorm{\Lts[t+1] - \Los} &= \infnorm{U_s\Lambda_s U_s^{\top} - \Los}= \infnorm{\sum_{0 \leq s, t < \infty} \Hs^{s} \Los U_s \Lambda_s^{-(s + t + 1)} U_s^{\top} \Los \Hs^{t} - \Los} \notag,\\
    &\leq \infnorm{\Los U_s \Lambda_s^{-1} U_s^{\top} \Los - \Los} + \sum_{1 \leq s + t < \infty} \infnorm{\Hs^{s} \Los U_s \Lambda_s^{-(s + t + 1)} U_s^{\top} \Los \Hs^{t}}.\label{eq:ltnloerr}
  \end{align}
  We separately bound the first and the second term of RHS. The first term can be bounded as follows:
  \begin{align}
    &\infnorm{\Los U_s \Lambda_s^{-1} U_s^{\top} \Los - \Los} \overset{(\zeta_{1})}{\leq} \infnorm{\Los \begin{bmatrix}0 & V\Sigma^{-1}U^{\top}\\ U\Sigma^{-1} V^{\top} & 0\end{bmatrix} \Los - \Los}\\
    &\leq \infnorm{\Lo V\Sigma^{-1}U^{\top}\Lo - \Lo} \overset{(\zeta_{2})}{\leq} \frac{\mu^{2}r}{\sqrt{mn}} \twonorm{\Lo U \Lambda^{-1} U^{\top} \Lo - \Lo} \overset{(\zeta_{3})}{\leq} \frac{\mu^{2}r}{\sqrt{mn}} \left(\abs{\sigma_{k + 1}^{*}} + 5\twonorm{H}\right), \label{fTerm}
  \end{align}
  where $(\zeta_{1})$ follows Remark \ref{rem:sym}, $(\zeta_{2})$ from Lemma \ref{lem:incohBound} and $(\zeta_{3})$ follows from Claim 1 of Lemma \ref{lem:geomBound}.

  We now bound second term of RHS of \eqref{eq:ltnloerr} which we again split in two parts. We first bound the terms with $s + t > \log n$:
  \begin{align*}
    \infnorm{\Hs^{s} \Los \Us \Lambda_s^{-(s + t + 1)} \Us^{\top} \Los \Hs^{t}} &\leq \twonorm{\Hs^{s} \Los \Us \Lambda_s^{-(s + t + 1)} \Us^{\top} \Los \Hs^{t}}\overset{(\zeta_{1})}{\leq} \twonorm{\Hs}^{s + t} 4 \left(\frac{2}{\sigma_{k}^{*}}\right)^{-(s + t - 1)} \\
    &\hspace*{-2.5cm}\leq 4 \twonorm{H} \left(\twonorm{H} \frac{2}{\sigma_{k}^{*}}\right)^{-(s + t - 1)}\ \  \overset{(\zeta_2)}{\leq}\ \  \frac{4\mu^{2}r}{m} \twonorm{H} \left(\frac{1}{2}\right)^{-(s + t - 1 - \log n)}, \numberthis \label{lTerm}
  \end{align*}
  where $(\zeta_{1})$ follows from the second claim of Lemma \ref{lem:geomBound} and noting that $\twonorm{H_s} = \twonorm{H}$ and $(\zeta_2)$ follows from assumption on $\twonorm{H}$ and using the fact that $s+t\geq \log n$.

  Summing up over all terms with $s + t > \log n$, we get from \ref{lTerm} and \ref{fTerm}:
  \begin{equation}
    \label{intRes}
    \infnorm{\Lts[t + 1] - \Los} \leq \frac{\mu^{2}r}{\sqrt{mn}} \left(\abs{\sigma^{*}_{k + 1}} + 20\twonorm{H}\right) + \sum_{0 < s + t \leq \log n} \infnorm{\Hs^{s} \Los \Us \Lambda_s^{-(s + t + 1)} \Us^{\top} \Los \Hs^{t}}
  \end{equation}

  Now, for terms corresponding to $1 \leq s + t \leq \log n$, we have:
  \begin{align*}
    &\infnorm{\Hs^{s} \Los \Us \Lambda_s^{-(s + t + 1)} \Us^{\top} \Los \Hs^{t}} = \max\limits_{q_1 \in [m + n], q_2 \in [m + n]} \abs{e_{q_1}^{\top} \Hs^{s} \Los \Us \Lambda_s^{-(s + t + 1)} \Us^{\top} \Los \Hs^{t} e_{q_2}} \\
    &\qquad \leq \left(\max\limits_{q_1 \in [m + n]} \twonorm{e_{q_1}^{\top} \Hs^{s} \Us^{*}}\right) \twonorm{\Sigma_s^{*} (U_s^{*})^{\top} \Us \Lambda_s^{-(s + t + 1)}\Us^{\top} U_s^{*} \Sigma_s^{*}}  \left(\max\limits_{q_2 \in [m + n]} \twonorm{e_{q_2}^{\top} H^{t} U_s^{*}}\right) \\
    &\qquad \overset{(\zeta_{1})}{\leq} \frac{\mu^{2}r}{m} \upsilon^{s + t} \twonorm{\Los \Us \Lambda_s^{-(s + t + 1)} \Us^{\top} \Los} \overset{(\zeta_{2})}{\leq} \frac{4\mu^{2}r}{m} \upsilon^{s + t} \left(\frac{2}{\sigma_{k}^{*}}\right)^{s + t - 1} \leq \frac{4\mu^{2}r}{m} \upsilon \left(\frac{1}{2}\right)^{s + t - 1}, \numberthis \label{mTerm}
  \end{align*}
  where $(\zeta_1)$ follows from assumption on $H$ in the lemma statement, $(\zeta_2)$ follows from Claim 2 of Lemma \ref{lem:geomBound}.

  It now remains to bound the terms, $\max\limits_{q_1 \in [m + n]} \twonorm{e_{q_1}^{\top} \Hs^{s} \Us^{*}}$. Note from Remark \ref{rem:sym}.1 that $\Us^* = \frac{1}{\sqrt{2}} \begin{bmatrix}  V^* & V^* \\ U^* & -U^* \end{bmatrix}$. Now, we have the following cases for $H^s_s$:

  \begin{center}
    \begin{tabular}{c c}
      $H_s^j = \begin{bmatrix} \left(H^\top H\right)^\frac{s}{2} & 0 \\ 0 & \left(HH^\top\right)^\frac{s}{2} \end{bmatrix}$ when $s$ is even
      &
      $H_s^j = \begin{bmatrix} 0 & H^\top\left(HH^\top\right)^{\lfloor \frac{s}{2} \rfloor}\\ H\left(H^\top H\right)^{\lfloor \frac{s}{2} \rfloor} & 0 \end{bmatrix}$ when $s$ is odd
    \end{tabular}
  \end{center}

  In these two cases, we have:

  \begin{center}
    \begin{tabular}{c c}
      $H_s^s \Us^* = \frac{1}{\sqrt{2}}\begin{bmatrix}
      \left(H^\top H\right)^\frac{s}{2}V^* & \left(H^\top H\right)^\frac{s}{2}V^* \\
      \left(HH^\top\right)^\frac{s}{2}U^* & -\left(HH^\top\right)^\frac{s}{2}U^* \end{bmatrix}$
      &
      $H_s^s \Us^* = \frac{1}{\sqrt{2}}\begin{bmatrix}
      H^\top\left(HH^\top\right)^{\lfloor \frac{s}{2} \rfloor}U^* & -H^\top\left(HH^\top\right)^{\lfloor \frac{s}{2} \rfloor}U^* \\
      H\left(H^\top H\right)^{\lfloor \frac{s}{2} \rfloor}V^* & H\left(H^\top H\right)^{\lfloor \frac{s}{2} \rfloor}V^* \end{bmatrix}$
    \end{tabular}
  \end{center}

  This leads to the following 4 cases for $\max\limits_{q_1 \in [m + n]} \twonorm{e_{q_1}^{\top} \Hs^{s} \Us^{*}}$:

  \begin{center}
    \begin{tabular}{c c c}
      for $s$ even &  $\max\limits_{q' \in [n]} \twonorm{e_{q'}^\top\left(H^\top H\right)^\frac{s}{2}V^*}$ & $\max\limits_{q' \in [m]} \twonorm{e_{q'}^\top\left(HH^\top\right)^\frac{s}{2}U^*}$ \\
      for $s$ odd  & $\max\limits_{q' \in [n]} \twonorm{e_{q'}^\top H^\top\left(HH^\top\right)^{\lfloor \frac{s}{2} \rfloor}U^*}$ & $\max\limits_{q' \in [m]} \twonorm{e_{q'}^\top H\left(H^\top H\right)^{\lfloor \frac{s}{2} \rfloor}V^*}$
    \end{tabular}
  \end{center}

  we get the bound on these terms in Lemma \ref{lem:grLemma}. Also, note from the Remark \ref{rem:sym}.2 that $\infnorm{\Los - \Lts[t + 1]} = \infnorm{\Lo - \Lt[t + 1]}$.

  Now, summing up \ref{mTerm} over all $1 \leq s + t \leq \log n$ and combining with \ref{intRes} we get the required result.
\end{proof}

In the next lemma, we show that with the threshold chosen in the algorithm, we show an improvement in the estimation of $\So$ by $\Sot[t + 1]$.

\begin{lemma}
  \label{lem:sProg}
  In the $t^{\textit{th}}$ iterate of the $q^{\textit{th}}$ stage, assume the following holds:
  \begin{enumerate}
    \item $\infnorm{\Lo - \Lt} \leq  \frac{2\mu^2r}{m} \left(\sigma^{*}_{k + 1} + \left(\frac{1}{2}\right)^{z}\sigma^{*}_{k}\right)$
    \item $\ddfrac{7}{8} \left(\sigma^{*}_{k + 1} + \left(\ddfrac{1}{2}\right)^{z}\sigma^{*}_{k}\right) \leq \left(\lambda_{k + 1} + \left(\ddfrac{1}{2}\right)^{z}\lambda_{k}\right) \leq \ddfrac{9}{8} \left(\sigma^{*}_{k + 1} + \left(\ddfrac{1}{2}\right)^{z}\sigma^{*}_{k}\right)$
  \end{enumerate}
  where $\sigma_k^*$ and $\sigma_{k+1}^*$ are the $k$ and $(k+1)^{\textit{th}}$ singular values of $L^*$, $\lambda_k$ and $\lambda_{k+1}$ are the $k$ and $(k+1)^{\textit{th}}$ singular values of $\Mt$ and, $r$ and $\mu$ are the rank and incoherence of the $m\times n$ matrix $L^*$ respectively. Then we have
  \begin{enumerate}
    \item $Supp\left(\Sot\right) \subseteq Supp\left(\So\right)$
    \item $\infnorm{\Sot - \So} \leq \frac{8\mu^{2}r}{m} \left(\sigma^{*}_{k + 1} + \left(\frac{1}{2}\right)^{z}\sigma^{*}_{k} \right)$
  \end{enumerate}
\end{lemma}

\begin{proof}
  We first prove the first claim of the lemma. Consider an index pair $(i,j) \notin Supp(\So)$.
  \[
  \abs{\M_{ij} - \Lt_{ij}} \leq \frac{2\mu^2r}{m} \left( \sigma^{*}_{k + 1} + \left(\frac{1}{2}\right)^{z}\sigma^{*}_{k} \right) \overset{(\zeta_1)}{\leq} \frac{16\mu^{2}r}{7m} \left( \lambda_{k + 1} + \left(\frac{1}{2}\right)^{z}\lambda_{k} \right) \leq \thresh \left( \lambda_{k + 1} + \left(\frac{1}{2}\right)^{z}\lambda_{k} \right)
  \]
  where $(\zeta_1)$ follows from the second assumption.
  Hence, we do not threshold any entry that is not corrupted by $\So$.

  Now, we prove the second claim of the lemma. Consider an index entry $(i,j) \in Supp(\So)$. Here, we consider two cases:
  \begin{enumerate}
    \item The entry $(i,j) \in Supp(\Sot)$: Here the entry $(i,j)$ is thresholded. We know that $\Lt_{ij} + \Sot_{ij} = \Lo_{ij} + \So_{ij}$ from which we get
    \[
    \abs{\Sot_{ij} - \So_{ij}} = \abs{\Lo_{ij} - \Lt_{ij}} \leq \infnorm{\Lo - \Lt}
    \]
    \item The entry $(i,j) \notin Supp(\Sot)$: Here the entry $(i,j)$ is not thresholded. We know that $\abs{\Lo_{ij} + \So_{ij} - \Lt_{ij}} \leq \zeta$ from which we get
    \begin{align*}
      \abs{\So_{ij}} &\leq \zeta + \abs{\Lo_{ij} - \Lt_{ij}} \\
      &\overset{(\zeta_2)}{\leq} \frac{36\mu^{2}r}{8m} \left( \sigma^{*}_{k + 1} + \left(\frac{1}{2}\right)^{z}\sigma^{*}_{k} \right) + \frac{2\mu^{2}r}{m} \left( \sigma^{*}_{k + 1} + \left(\frac{1}{2}\right)^{z}\sigma^{*}_{k} \right) \\
      &\leq \frac{8\mu^{2}r}{m} \left(\sigma^{*}_{k + 1} + \left(\frac{1}{2}\right)^{z}\sigma^{*}_{k} \right)
    \end{align*}
    where $(\zeta_2)$ follows from the second assumption along with the assumption about $\eta = \frac{\mu^2r}{m}$.
  \end{enumerate}

  The above two cases prove the second statement of the lemma.
\end{proof}

\begin{lemma}
  \label{lem:numInItr}
  The number of iteration $T$ in the inner loop of Algorithm \ref{alg:gca} and Algorithm \ref{alg:sap} satisfy:
  \[T \geq 10\log \left(7 n^2 \mu^2 r \sigma_1^*/\epsilon\right)\]
  w.p $\geq 1 - n^{-(10 + \log \alpha)}$. Here $\sigma_1^*$ is the highest singular value of $L^*$, $r$ is it's rank and $\mu$ is it's incoherence.
\end{lemma}
\begin{proof}
  We have the bound since
  \begin{align*}
    \twonorm{\frac{n_{1}n_{2}}{\abs{\Omega}} \Pom\left(\M-\Sot[0]\right)} &= \twonorm{\Lo + \left(\mathcal{I} - \frac{\Pom}{p}\right) \left(\left(\Sot[0] - S^{*}\right)  - \Lo\right) + \left(S^{*} - \Sot[0]\right)} \\
    &\geq \sigma_{1}^{*} - \twonorm{H} \overset{(\zeta_1)}{\geq} \frac{3}{4} \sigma_{1}^{*}
  \end{align*}
  where $(\zeta_1)$ follows from Lemma \ref{lem:hBound}.
\end{proof}

We will now prove Lemma~\ref{lem:einf} \\
\begin{proofof}{Lemma~\ref{lem:einf}}
  Recall the definitions of $E_{1} = \left(\So - \Sot\right)$ , $E_{2} = \left(\Lt - \Lo\right)$ , $E_{3} = \left(\mathcal{I} - \frac{\Pomqt}{p}\right) \left(E_{2} - E_{1}\right)$ and $\beta = 2\sqrt{\frac{n}{p}} \infnorm{E_{2} - E_{1}}$.
  Recall that $H \coloneqq E_1 + E_3$
  From Lemma~\ref{lem:satDefn}, we have that $\frac{1}{\beta}E_{3}$ satisfies Definition~\ref{def:mzero}. This implies that the matrix $\frac{1}{\beta}\left(E_{1} + E_{3}\right)$ satisfies the conditions of Lemma~\ref{lem:grLemma}. Now, we have $\forall 1 \leq a \leq \lceil \log n \rceil$ and $\forall i \in [n]$:
  \begin{multline*}
    \twonorm{e_{i} (HH^{\top})^a U^{*}} = \beta^{2a} \twonorm{e_{i}  \left(\left(\frac{1}{\beta}H\right)\left(\frac{1}{\beta}H\right)^\top \right)^a U^{*}} \\ \overset{(\zeta)}{\leq} \beta^{2a} \left( \frac{\rho n}{\beta} \infnorm{E_{1}} + c \log n \right)^{2a} \mu \sqrt{\frac{r}{m}} \leq \mu \sqrt{\frac{r}{m}} \left( \rho n \infnorm{E_{1}} + 2c \sqrt{\frac{n}{p}} \left(\infnorm{E_{1}} + \infnorm{E_{2}}\right)\log n\right)^{2a}
  \end{multline*}
  where $(\zeta)$ follows from the application of Lemma~\ref{lem:grLemma} along with the incoherence assumption on $U^{*}$. The other statements of the lemma can be proved in a similar manner by invocations of the different claims of Lemma~\ref{lem:grLemma}.
\end{proofof}


\subsection{Algorithm \ncgca}

\begin{proofof}{Theorem \ref{thm:gca}} From Lemma \ref{lem:numInItr} we know that $T\geq\log(\frac{3\mu^2r\sigma_1^*}{\epsilon})$. Consider the stage $q$ reached at the termination of the algorithm. We know from Lemma \ref{lem:tProg} that:

\begin{enumerate}
  \item $\infnorm{E^{(T)}} \leq \frac{8\mu^2r}{m} \left(\sigma^{*}_{k_{q} + 1} + \left(\frac{1}{2}\right)^{T}\sigma^{*}_{k_{q}}\right) \leq \frac{8\mu^2r}{m} \sigma^{*}_{k_{q} + 1} + \frac{\epsilon}{10n}$
  \item $\infnorm{L^{(T)} - L^{*}} \leq \frac{2\mu^2r}{m} \left(\sigma^{*}_{k_{q} + 1} + \left(\frac{1}{2}\right)^{T}\abs{\sigma^{*}_{k_{q}}}\right) \leq \frac{2\mu^2r}{m} \sigma^{*}_{k_{q} + 1} + \frac{\epsilon}{10n}$
\end{enumerate}

Combining this with Lemmas \ref{lem:weyl-perturbation} and \ref{lem:hBound}, we get:

\begin{equation}
  \label{eqn:term}
  \abs{\sigma_{k_{q} + 1}(M^{T})} \geq \sigma^{*}_{k_{q} + 1} - \frac{1}{100} \left(\sigma^{*}_{k_{q} + 1} + \frac{m\epsilon}{10n\mu^2r}\right)
\end{equation}


When the while loop terminates, $\thresh \sigma_{k_q+1}\left(M^{(T)}\right) < \frac{\epsilon}{2 n}$, which from \ref{eqn:term}, implies that $\sigma_{k_q + 1}^{*} < \frac{m\epsilon}{7n\mu^{2}r}$. So we have:
\begin{align*}
  \infnorm{L - \Lo} = \infnorm{\Lt[T] - \Lo}
  \leq \frac{2 \mu^2 r}{m} \sigma_{k_{q}+1}^* + \frac{\epsilon}{10n} \leq \frac{\epsilon}{2n}.
\end{align*}

We will now bound the number of iterations required for the \ncgca\ to converge.

From claim 2 of Lemma~\ref{lem:gTrack}, we have $\sigma^{*}_{k_{q} + 1} \leq \frac{17}{32}\sigma^{*}_{k_{q-1} + 1} \quad \forall q \geq 1$. By recursively applying this inequality, we get $\sigma^{*}_{k_{q} + 1} \leq \left(\frac{17}{32}\right)^{q} \sigma^{*}_{1}$. We know that when the algorithm terminates, $ \sigma_{k_q + 1}^{*} < \frac{\epsilon}{7\mu^{2}r} $. Since, $\left(\frac{17}{32}\right)^{q} \sigma^{*}_{1}$ is an upper bound for $\sigma_{k_q + 1}^{*}$, an upper bound for the number of iterations is $5 \log \left(\frac{7\mu^2r\sigma_{1}^{*}}{\epsilon}\right)$. Also, note that an upper bound to this quantity is used to partition the samples provided to the algorithm. This happens with probability $\geq 1 - T^2n^{-(10+\log\alpha)}\geq 1 - n^{-\log\alpha}$. This concludes the proof.
\end{proofof}

In the following lemma, we show that we make progress simultaneously in the estimation of both $\So$ and $\Lo$ by $\Sot$ and $\Lt$. We make use of Lemmas \ref{lem:Lprog} and \ref{lem:sProg} to show progress in the estimation of one affects the other alternatively. We also emphasize the roles of the following quantities in enabling us to prove our convergence result:
\begin{enumerate}
  \item $\twonorm{H}$ - We use Lemma \ref{lem:hBound} to bound this quanitity
  \item The analysis of the following 4 quanitities is crucial to obtaining error bounds in $\lVert\rVert_{\infty}$ norm
  \begin{center}
    \begin{tabular}{c c c}
      for $j$ even &  $\max\limits_{q' \in [n]} \twonorm{e_{q'}^\top\left(H^\top H\right)^\frac{j}{2}V^*}$ & $\max\limits_{q' \in [m]} \twonorm{e_{q'}^\top\left(HH^\top\right)^\frac{j}{2}U^*}$ \\
      for $j$ odd  &  $\max\limits_{q' \in [n]} \twonorm{e_{q'}^\top H^\top\left(HH^\top\right)^{\lfloor \frac{j}{2} \rfloor}U^*}$ & $\max\limits_{q' \in [m]} \twonorm{e_{q'}^\top H\left(H^\top H\right)^{\lfloor \frac{j}{2} \rfloor}V^*}$
    \end{tabular}
  \end{center}
  We use Lemma \ref{lem:einf} to bound this quantity.
\end{enumerate}

\begin{lemma}
  \label{lem:tProg}
  Let $\Lo$, $\Omega$, $\So$ and $\Sot$ satisfy Assumptions 1,2,3 respectively. Then, in the $t^{\textrm{th}}$ iteration of the $q^{\textrm{th}}$ stage of Algorithm~\ref{alg:gca}, $\Sot$ and $\Lt$ satisfy:
  \begin{align*}
    \infnorm{\Sot - \So} &\leq \frac{8\mu^2 r}{m} \left(\abs{\sigma_{k_{q}+1}^*}+\left(\frac{1}{2}\right)^{t-3}\abs{\sigma_{k_{q}}^*}\right), \\
    \supp{\Sot} &\subseteq \supp{\So}, \mbox{ and } \\
    \infnorm{\Lt-\Lo} &\leq \frac{2\mu^2 r}{m} \left(\abs{\sigma_{k_{q}+1}^*}+\left(\frac{1}{2}\right)^{t-3}\abs{\sigma_{k_{q}}^*}
    \right).
  \end{align*}
  with probability $ \geq 1 - ((q-1)T+t-1)n^{-(10+\log\alpha)}$ where $T$ is the number of iterations in the inner loop.
\end{lemma}

\begin{proof}
  We prove the lemma by induction on both $q$ and $t$.

  \textbf{Base Case: $q=1$ and $t=0$} \\
  We begin by first proving an upper bound on $\infnorm{\Lo}$. We do this as follows:
  \[
  \abs{\Lo_{ij}} = \abs{\sum\limits_{k = 1}^{r} \sigma_{k}^{*} u^{*}_{ik} v^{*}_{jk}} \leq \sum\limits_{k = 1}^{r} \sigma_{k}^{*} \abs{u^{*}_{ik} v^{*}_{jk}} \leq \sigma_{1}^{*} \sum\limits_{k = 1}^{r}  \abs{u^{*}_{ik} v^{*}_{jk}} \leq \frac{\mu^{2}r}{\sqrt{mn}} \sigma_{1}^{*}
  \]
  where the last inequality follows from Cauchy-Schwartz and the incoherence of $U^{*}$. This directly proves the third claim of the lemma for the base case.
  We also note that due to the thresholding step and the incoherence assumption on $\Lo$, we have:
  \begin{enumerate}
    \item $\infnorm{\E^{(0)}} \leq \frac{8\mu^2 r}{m}\left(\sigma_2^* + 2 \sigma_1^*\right) \overset{(\zeta)}{\leq} \frac{8\mu^2 r}{m}\left( 8 \sigma_{k_{1}}^*\right), \mbox{ and }$
    \item $\supp{\Sot} = \supp{\So}.$
  \end{enumerate}
  where $(\zeta)$ follows from Lemma \ref{lem:gTrack}.
  So the base case of induction is satisfied.

  \textbf{Induction over $t$} \\
  We first prove the inductive step over $t$ (for a fixed $q$).
  By inductive hypothesis we assume that:
  \begin{enumerate}
    \item[a)] $\infnorm{\Et} \leq \frac{8\mu^2 r}{m}\left(\sigma_{k_{q}+1}^*+\left(\frac{1}{2}\right)^{t-3}\sigma_{k_{q}}^*\right)$
    \item[b)] $\supp{\Sot}\subseteq \supp{\So}$.
    \item[c)] $\infnorm{\Lo - \Lt} \leq \frac{2\mu^2 r}{m}\left(\sigma_{k_{q}+1}^*+\left(\frac{1}{2}\right)^{t-3}\sigma_{k_{q}}^* \right)$
  \end{enumerate}

  with probability $1-((q-1)T+t-1)n^{-(10+\log\alpha)}$. Then by Lemma~\ref{lem:Lprog}, we have:
  \begin{equation}
    \label{eqn:lprog}
    \infnorm{\Ltn-\Lo} \leq \frac{\mu^2 r}{m}\left(\sigma_{k_{q}+1}^*+20\twonorm{H}+8\upsilon\right)
  \end{equation}

  From Lemma \ref{lem:einf}, we have:

  \begin{equation}
    \label{eqn:upsilon}
    \upsilon \leq \rho n \infnorm{\Et} + 8\beta\alpha\log n \overset{(\zeta_{1})}{\leq} \frac{1}{100} \left(\sigma^{*}_{k_{q} + 1} + \left(\frac{1}{2}\right)^{t-3}\sigma_{k_{q}}^*\right) + 8\beta\alpha\log n \overset{(\zeta_{2})}{\leq} \frac{1}{50} \left(\sigma^{*}_{k_{q} + 1} + \left(\frac{1}{2}\right)^{t-3}\sigma_{k_{q}}^*\right)
  \end{equation}

  where $(\zeta_{1})$ follows from our assumptions on $\rho$ and our inductive hypothesis on $\infnorm{\Et}$ and $(\zeta_{2})$ follows from our assumption on $p$ and by noticing that $ \infnorm{D} \leq \infnorm{\Et} + \infnorm{\Lo - \Lt}$. Recall that $D = \Lt-\Lo+\Sot-\So$.

  From Lemma \ref{lem:hBound}:
  \begin{equation}
    \label{eqn:twoNorm}
    \twonorm{H} \leq \frac{1}{100} \left(\sigma^{*}_{k_{q} + 1} + \left(\frac{1}{2}\right)^{t-3}\sigma_{k_{q}}^*\right)
  \end{equation}

  with probability $\geq 1-n^{-(10+\log\alpha)}$. From Equations \ref{eqn:twoNorm}, \ref{eqn:upsilon} and \ref{eqn:lprog}, we have:
  \[
  \infnorm{\Lo - \Ltn} \leq \frac{2\mu^2 r}{m}\left(\sigma_{k_{q}+1}^*+\left(\frac{1}{2}\right)^{t-2}\sigma_{k_{q}}^*\right)
  \]

  which by union bound holds with probability $\geq 1 - ((q-1)T+t)n^{-(10+\log\alpha)}$. Hence, using Lemma \ref{lem:sProg} and \ref{lem:eigTrack} we have:

\begin{enumerate}
\item $\infnorm{\Etn} \leq  \frac{8\mu^2 r}{m}\left(\sigma_{k_{q}+1}^*+\left(\frac{1}{2}\right)^{t-2}\sigma_{k_{q}}^*
    \right)$
\item $\supp{\Sot{t + 1}} \subseteq \supp{\So}$.
\end{enumerate}
  which also holds with probability $\geq 1 - ((q-1)T+t)n^{-(10+\log\alpha)}$. This concludes the proof for induction over $t$.

  \textbf{Induction Over Stages $q$} \\
  We now prove the induction over $q$. Suppose the hypothesis holds for stage $q$.
  At the end of stage $q$, we have:
  \begin{enumerate}
    \item	$\infnorm{\E^{(T)}} \leq \frac{8\mu^2 r}{m}\left(\sigma_{k_{q}+1}^*+\left(\frac{1}{2}\right)^{T}\sigma_{k_{q}}^*\right)
    \leq \frac{8\mu^2 r \sigma_{k_{q}+1}^*}{m} + \frac{\epsilon}{10 n}$, and
    \item	$\supp{\Sot[T]} \subseteq \supp{\So}$.
  \end{enumerate}
  with probability $\geq 1 - (qT - 1)n^{-(10+\log\alpha)}$. From Lemmas ~\ref{lem:weyl-perturbation} and ~\ref{lem:hBound} we get:
  \begin{equation}
    \abs{\sigma_{k_{q}+1}\left(M^{(T)}\right) - \sigma_{k_{q}+1}^*} \leq \twonorm{H} \leq \frac{1}{100} \left(\sigma_{k_{q}+1}^* + \frac{m\epsilon}{10n\mu^2r} \right)
    \label{eqn:nStage}
  \end{equation}

  with probability $1 - n^{-(10+\log\alpha)}$. We know that $\thresh \sigma_{k_{q}+1}\left(\Mt\right) \geq \frac{\epsilon}{2 n}$ which with \ref{eqn:nStage} implies that $\abs{\sigma_{k_{q}+1}^*} > \frac{m\epsilon}{10 n\mu^2 r}$.

  \begin{align*}
    \infnorm{\Lt[T + 1] - \Lo} &\leq \frac{2\mu^2 r}{m}\left(\sigma_{k_{q}+1}^*+\left(\frac{1}{2}\right)^{T + 1}\sigma_{k_{q}}^* \right) \leq \frac{2\mu^2 r}{m}\left(\sigma_{k_{q}+1}^* + \frac{m\epsilon}{20n\mu^2 r n} \right) \\
    &\leq \frac{2\mu^2 r}{m}\left(\sigma_{k_{q}+1}^* + \frac{\sigma_{k_{q}+1}^*}{2} \right) \leq \frac{2\mu^2 r}{m}\left(2 \sigma_{k_{q}+1}^* \right) \overset{(\zeta_{4})}{\leq} \frac{2\mu^2 r}{m} \left(8 \sigma_{k_{q + 1}}^* \right)
  \end{align*}
  where $(\zeta_{4})$ follows from Lemma \ref{lem:gTrack}. By union bound this holds with probability $\geq 1 - qTn^{-(10+\log\alpha)}$.

  Now, from \ref{lem:eigTrack} and \ref{lem:sProg}, we have through a similar series of arguments as above:
  \begin{equation}
    \infnorm{\Et{T + 1}} \leq \frac{8\mu^2 r}{m} \left(8 \sigma_{k_{q + 1}}^*\right)
  \end{equation}

 which holds with probability $\geq 1 - qTn^{-(10+\log\alpha)}$.
\end{proof}

\begin{lemma}
  \label{lem:gTrack}
  Suppose at the beginning of the $q^\text{th}$ stage of algorithm \ref{alg:gca}:
  \begin{enumerate}
     \item $\norm{L^*-L^{(0)}}_\infty\leq\ddfrac{2\mu^2r}{m}\left(2\sigma_{k_{q - 1}+1}^*\right)$
    \item $\norm{E^{(0)}}_\infty\leq\ddfrac{8\mu^2r}{m}\left(2\sigma_{k_{q - 1}+1}^*\right)$
   \end{enumerate}

  Then, the following hold:
  \begin{enumerate}
      \item $\sigma^{*}_{k_{q}}\geq\frac{15}{32}\sigma^*_{k_{q-1}+1}$
    \item $\sigma^{*}_{k_{q}+1}\leq\frac{17}{32}\sigma^*_{k_{q-1}+1}$
  \end{enumerate}

  with probability $\geq 1 - n^{-(10+\log\alpha)}$
\end{lemma}

\begin{proof}
  We know that:
  \[
  \lambda_{k_{q}} \leq \sigma^{*}_{k_{q}} + \twonorm{H}, \qquad \lambda_{k_{q-1} + 1} \geq \sigma^{*}_{k_{q-1} + 1} - \twonorm{H}, \qquad \lambda_{k_{q}} \geq \frac{\lambda_{k_{q-1} + 1}}{2}
  \]

  Combining the three inequalities, we get:
  \[
  \sigma^{*}_{k_{q}} \geq \frac{\sigma^{*}_{k_{q-1} + 1} - 3 \twonorm{H}}{2}
  \]

  Applying Lemma \ref{lem:hBound}, we get the first claim of the lemma.

  Similar to the first claim, we have:
  \[
  \lambda_{k_{q} + 1} \geq \sigma^{*}_{k_{q} + 1} - \twonorm{H}, \qquad \lambda_{k_{q-1} + 1} \leq \sigma^{*}_{k_{q-1} + 1} + \twonorm{H}, \qquad \lambda_{k_{q} + 1} \leq \frac{\lambda_{k_{q-1} + 1}}{2}
  \]

  Again, combining the three inequalities, we get:
  \[
  \sigma^{*}_{k_{q} + 1} \leq \frac{\sigma^{*}_{k_{q-1} + 1} + 3 \twonorm{H}}{2}
  \]

  Another application of Lemma \ref{lem:hBound} gives the second claim.
\end{proof}

\subsection{Algorithm \ncralgo}
\label{app:rank}

\floatname{algorithm}{Algorithm}

\begin{algorithm}[t]
  \caption{$\widehat{L} ~ =$ \ncralgo$(\Omega, \Pom(M), \epsilon, r, \thresh, \sigma)$: Non-convex Robust Matrix Completion}
  \label{alg:sap}
  \begin{algorithmic}[1]
    \STATE {\bf Input}: Observed entries $\Omega$, Matrix $\Pom(M) \in \R^{m\times n}$, convergence criterion $\epsilon$, target rank $r$, thresholding parameter $\thresh$, upper bound on $\sigma^*_1\ \sigma$
    \STATE $T\leftarrow 10\log{\frac{10\mu^2 r \sigma}{\epsilon}}$ \hspace*{1cm}\hfill{\em /*Number of inner iterations*/}
    \STATE Partition $\Omega$ into $rT + 1$ subsets $\{\Omega_0\} \cup \{\Omega_{q, t} : q\in[r], t\in[T]\}$ using \ref{alg:split}
    \STATE $L^{(0)}=0, \, \zeta \, \leftarrow \, \thresh \sigma$
    \STATE $M^{(0)} \, \leftarrow \, \frac{mn}{\abs{\Omega_0}} \mathcal{P}_{\Omega_0}(M - \HT(M))$
    \STATE $q \, \leftarrow \, 0$
    \WHILE{$\sigma_{q + 1}(M^{(0)}) > \frac{\epsilon}{2\thresh m}$}
    \STATE $q \, \leftarrow \, q + 1$
    \FOR{Iteration $t=0$ to $t=T$}
    \STATE $S^{(t)}=H_{\zeta}(\mathcal{P}_{\Omega_{q, t}}(M-L^{(t)}))$ \hspace*{1cm}\hfill{\em /*Projection onto set of sparse matrices*/}
    \STATE $M^{(t)} = L^{(t)}-\frac{mn}{\abs{\Omega_{q, t}}}\mathcal{P}_{\Omega_{q, t}}( L^{(t)} + S^{(t)} - M)$ \hspace*{1cm}\hfill{\em /*Gradient Descent Update*/}
    \STATE $L^{(t+1)}=P_q(M^{(t)})$ \hspace*{1cm}\hfill{\em /*Projected Gradient Descent step*/}
    \STATE Set threshold $\zeta \, \leftarrow \, \thresh\, \left(\sigma_{q+1}(M^{(t)}) +\left(\frac{1}{2}\right)^{t} \sigma_q(M^{(t)})\right)$
    \ENDFOR
    \STATE $S^{(0)}=S^{(T)}, L^{(0)}=L^{(T + 1)}, M^{(0)} = M^{(T)}$
    \ENDWHILE
    \STATE {\bf Return: }$L^{(T + 1)}$
  \end{algorithmic}
\end{algorithm}

\begin{proofof}{Theorem \ref{thm:sap}}
From Lemma \ref{lem:numInItr} we know that $T\geq\log(\frac{3\mu^2nr\sigma_1^*}{m\epsilon})$.

Consider the stage $q$ reached at the termination of the algorithm. We know from Lemma \ref{lem:s:tProg} that:

\begin{enumerate}
  \item $\infnorm{E^{(T)}} \leq \frac{8\mu^2r}{m} \left(\sigma^{*}_{q + 1} + \left(\frac{1}{2}\right)^{T}\sigma^{*}_{q}\right) \leq \frac{8\mu^2r}{m} \sigma^{*}_{q + 1} + \frac{\epsilon}{10n}$
  \item $\infnorm{L^{(T)} - L^{*}} \leq \frac{2\mu^2r}{m} \left(\sigma^{*}_{q + 1} + \left(\frac{1}{2}\right)^{T}\sigma^{*}_{q}\right) \leq \frac{2\mu^2r}{m} \sigma^{*}_{q + 1} + \frac{\epsilon}{10n}$
\end{enumerate}

Combining this with Lemmas \ref{lem:weyl-perturbation} and \ref{lem:hBound}, we get:

\begin{equation}
  \label{eqn:s:term}
  \sigma_{q + 1}(M) \geq \sigma^{*}_{q + 1} - \frac{1}{100} \left(\sigma^{*}_{q + 1} + \frac{m\epsilon}{10n\mu^2r}\right)
\end{equation}

When the while loop terminates, $\thresh \sigma_{q+1}\left(M^{(T)}\right) < \frac{\epsilon}{2 n}$, which from \ref{eqn:s:term}, implies that $\sigma_{q + 1}^{*} < \frac{m\epsilon}{7n\mu^{2}r}$. So we have:
\begin{align*}
  \infnorm{L - \Lo} = \infnorm{\Lt[T] - \Lo}
  \leq \frac{2 \mu^2 r}{m} |\sigma_{k_{q}+1}^*| + \frac{\epsilon}{10n} \leq \frac{\epsilon}{2n}.
\end{align*}
\end{proofof}

As in the case of the proof of Theorem~\ref{thm:gca}, the following lemma shows that we simultaneously make progress in both the estimation of $\Lo$ and $\So$ by $\Lt$ and $\Sot$ respectively. Similar to Lemma~\ref{lem:tProg}, we make use of Lemmas \ref{lem:sProg} and \ref{lem:Lprog} to show how improvement in estimation of one of the quantities affects the other and the other five terms, $\twonorm{H}$, $\max\limits_{q' \in [n]} \twonorm{e_{q'}^\top\left(H^\top H\right)^jV^*}$, $\max\limits_{q' \in [m]} \twonorm{e_{q'}^\top\left(HH^\top\right)^jU^*}$, $\max\limits_{q' \in [n]} \twonorm{e_{q'}^\top H^\top\left(HH^\top\right)^{j}U^*}$ and $\max\limits_{q' \in [m]} \twonorm{e_{q'}^\top H\left(H^\top H\right)^{j}V^*}$ are analyzed the same way:

\begin{lemma}\label{lem:s:tProg}
  Let $\Lo$, $\Omega$, $\So$ and $\Sot$ satisfy Assumptions 1,2,3 respectively. Then, in the $t^{\textrm{th}}$ iteration of the $q^{\textrm{th}}$ stage of Algorithm~\ref{alg:sap}, $\Sot$ and $\Lt$ satisfy:
  \begin{align*}
    \infnorm{\Sot - \So} &\leq \frac{8\mu^2 r}{m} \left(\sigma_{q+1}^*+\left(\frac{1}{2}\right)^{t-1}\sigma_{q}^*\right), \\
    \supp{\Sot} &\subseteq \supp{\So}, \mbox{ and } \\
    \infnorm{\Lt-\Lo} &\leq \frac{2\mu^2 r}{m} \left(\sigma_{q+1}^*+\left(\frac{1}{2}\right)^{t-1}\sigma_{q}^*
    \right).
  \end{align*}
  with probability $ \geq 1 - ((q-1)T+t-1)n^{-(10+\log\alpha)}$ where $T$ is the number of iterations in the inner loop.
\end{lemma}

\begin{proof}
  We prove the lemma by induction on both $q$ and $t$.

  \textbf{Base Case: $q=1$ and $t=0$} \\
  We begin by first proving an upper bound on $\infnorm{\Lo}$. We do this as follows:
  \[
  \abs{\Lo_{ij}} = \abs{\sum\limits_{k = 1}^{r} \sigma_{k}^{*} u^{*}_{ik} v^{*}_{jk}} \leq \sum\limits_{k = 1}^{r} \abs{\sigma_{k}^{*} u^{*}_{ik} v^{*}_{jk}} \leq \sigma_{1}^{*} \sum\limits_{k = 1}^{r}  \abs{u^{*}_{ik} v^{*}_{jk}} \leq \frac{\mu^{2}r}{m} \sigma_{1}^{*}
  \]
  where the last inequality follows from Cauchy-Schwartz and the incoherence of $U^{*}$. This directly proves the third claim of the lemma for the base case.
  We also note that due to the thresholding step and the incoherence assumption on $\Lo$, we have:
  \begin{enumerate}
    \item $\infnorm{\E^{(0)}} \leq \frac{8\mu^2 r}{m}\left(\sigma_2^* + 2 \sigma_1^*\right)$
    \item $\supp{\Sot} = \supp{\So}.$
  \end{enumerate}
  So the base case of induction is satisfied.

  \textbf{Induction over $t$} \\
  We first prove the inductive step over $t$ (for a fixed $q$).
  By inductive hypothesis we assume that:
  \begin{enumerate}
    \item[a)] $\infnorm{\Et} \leq \frac{8\mu^2 r}{m}\left(|\sigma_{q+1}^*|+\left(\frac{1}{2}\right)^{t-1}|\sigma_{q}^*|
    \right)$
    \item[b)] $\supp{\Sot}\subseteq \supp{\So}$.
    \item[c)] $\infnorm{\Lo - \Lt} \leq \frac{2\mu^2 r}{m}\left(|\sigma_{q+1}^*|+\left(\frac{1}{2}\right)^{t-1}|\sigma_{q}^*|
    \right)$
  \end{enumerate}
  with probability $1-((q-1)T+t-1)n^{-(10+\log\alpha)}$.

  Then by Lemma~\ref{lem:Lprog}, we have:
  \begin{equation}
    \label{eqn:s:lprog}
    \infnorm{\Ltn-\Lo} \leq \frac{\mu^2 r}{m}\left(|\sigma_{k_{q}+1}^*|+20\twonorm{H}+8\upsilon\right)
  \end{equation}

  From Lemma \ref{lem:einf}, we have:
  \begin{equation}
    \label{eqn:s:upsilon}
    \upsilon \leq \rho n \infnorm{\Et} + 8\beta\alpha\log n \overset{(\zeta_{1})}{\leq} \frac{1}{100} \left(\sigma^{*}_{q + 1} + \left(\frac{1}{2}\right)^{t-1}\sigma_{q}^*\right) + 8\beta\alpha\log n \overset{(\zeta_{2})}{\leq} \frac{1}{50} \left(\sigma^{*}_{q + 1} + \left(\frac{1}{2}\right)^{t-1}\sigma_{q}^*\right)
  \end{equation}

  where $(\zeta_{1})$ follows from our assumptions on $\rho$ and our inductive hypothesis on $\infnorm{\Et}$ and $(\zeta_{2})$ follows from our assumption on $p$ and by noticing that $ \infnorm{D} \leq \infnorm{\Et} + \infnorm{\Lo - \Lt}$. Recall that $D = \Lt-\Lo+\Sot-\So$.

  From Lemma \ref{lem:hBound}:
  \begin{equation}
    \label{eqn:s:twoNorm}
    \twonorm{H} \leq \frac{1}{100} \left(\sigma^{*}_{q + 1} + \left(\frac{1}{2}\right)^{t-1}\sigma_{q}^*\right)
  \end{equation}

  with probability $\geq 1-n^{-(10+\log\alpha)}$. From Equations \ref{eqn:s:twoNorm}, \ref{eqn:s:upsilon} and \ref{eqn:s:lprog}, we have:
  \[
  \infnorm{\Lo - \Ltn} \leq \frac{2\mu^2 r}{m}\left(\sigma_{q+1}^*+\left(\frac{1}{2}\right)^{t}\sigma_{q}^*\right)
  \]

  which by union bound holds with probability $\geq 1 - ((q-1)T+t)n^{-(10+\log\alpha)}$. Hence, using Lemma \ref{lem:sProg} and \ref{lem:eigTrack} we have:

  \begin{enumerate}
    \item	$\infnorm{\Etn} \leq  \frac{8\mu^2 r}{m} \left(\sigma_{q+1}^*+\left(\frac{1}{2}\right)^{t}\sigma_{q}^*
    \right)$
    \item	$\supp{\Sot[t + 1]} \subseteq \supp{\So}$.
  \end{enumerate}
  which also holds with probability $\geq 1 - ((q-1)T+t)n^{-(10+\log\alpha)}$. This concludes the proof for induction over $t$.

  \textbf{Induction Over Stages $q$} \\
  We now prove the induction over $q$. Suppose the hypothesis holds for stage $q$.
  At the end of stage $q$, we have:
  \begin{enumerate}
    \item	$\infnorm{\E^{(T)}} \leq \frac{8\mu^2 r}{m} \left(\sigma_{q+1}^*+\left(\frac{1}{2}\right)^{T}\sigma_{q}^*\right)
    \leq \frac{8\mu^2 r \sigma_{q+1}^*}{m} + \frac{\epsilon}{10 n}$
    \item	$\supp{\Sot[T]} \subseteq \supp{\So}$.
  \end{enumerate}
  with probability $\geq 1 - (qT - 1)n^{-(10+\log\alpha)}$.

  From Lemmas ~\ref{lem:weyl-perturbation} and ~\ref{lem:hBound} we get:
  \begin{equation}
    \abs{\sigma_{q+1}\left(M^{(T)}\right) - \sigma_{q+1}^*} \leq \twonorm{H} \leq \frac{1}{100} \left(\sigma_{q+1}^* + \frac{m\epsilon}{10n\mu^2r} \right)
    \label{eqn:s:nStage}
  \end{equation}

  with probability $1 - n^{-(10+\log\alpha)}$. We know that $\thresh \sigma_{q+1}\left(\Mt\right) \geq \frac{\epsilon}{2 n}$ which with \ref{eqn:s:nStage} implies that $\sigma_{q+1}^* > \frac{m\epsilon}{10 n\mu^2 r}$.

  \begin{align*}
    \infnorm{\Lt[T + 1] - \Lo} &\leq \frac{2\mu^2 r}{m}\left(\sigma_{q+1}^*+\left(\frac{1}{2}\right)^{T + 1}\sigma_{q}^* \right)\leq \frac{2\mu^2 r}{m}\left(\sigma_{q+1}^* + \frac{m\epsilon}{20\mu^2 r n} \right) \\
    &\leq \frac{2\mu^2 r}{m}\left(\sigma_{q+1}^* + \frac{\sigma_{q+1}^*}{2} \right)\leq \frac{2\mu^2 r}{m}\left(2 \sigma_{q+1}^* \right)
  \end{align*}
  By union bound this holds with probability $\geq 1 - qTn^{-(10+\log\alpha)}$.

  Now, from \ref{lem:eigTrack} and \ref{lem:sProg}, we have through a similar series of arguments as above:
  \begin{equation}
    \infnorm{\Et[T + 1]} \leq \frac{8\mu^2 r}{m} \left(2 \sigma_{k_{q + 1}}^*\right)
  \end{equation}

 which holds with probability $\geq 1 - qTn^{-(10+\log\alpha)}$.
\end{proof}

\subsection{Proof of a generalized form of Lemma \ref{lem:einf}}

\begin{lemma}
  \label{lem:grLemma}
  Suppose $H = H_{1} + H_{2}$ and $H\in\R^{m\times n}$ where $H_{1}$ satisfies Definition \ref{def:mzero} (Definition 7 from \cite{DBLP:conf/colt/0002N15}) and $H_{2}$ is a matrix with column and row sparsity $\rho$. Let $U$ be a matrix with rows denoted as $u_{1}, \dots, u_{m}$ and let $V$ be a matrix with rows denoted as $v_1, \dots, v_n$. Let $e_q$ be the $q^{\textit{th}}$ vector from standard basis. Let $\tau = \max\{\max\limits_{i \in [m]} \norm{u_{i}}, \max\limits_{i \in [n]} \norm{v_{i}}\}$. Then, for $0 \leq a \leq \log{n}$:

\begin{align*}
     \max\limits_{q \in [n]} \twonorm{e_{q}^\top\left(H^\top H\right)^{a}V}, \max\limits_{q \in [m]} \twonorm{e_{q}^\top\left(HH^\top\right)^aU} &\leq (\rho n \infnorm{H_2} + c \log n)^{2a} \tau \\
     \max\limits_{q \in [n]} \twonorm{e_{q}^\top H^\top\left(HH^\top\right)^{a}U}, \max\limits_{q \in [m]} \twonorm{e_{q}^\top H\left(H^\top H\right)^{a}V} &\leq (\rho n \infnorm{H_2} + c \log n)^{2a + 1} \tau \\
\end{align*}

with probability $n^{-2\log{\frac{c}{4}} + 4}$.
\end{lemma}

\begin{proof}

  Similar to \cite{DBLP:conf/colt/0002N15}, we will prove the statement for $q = 1$ and it can be proved for $q \in [n]$ by taking a union bound over all $q$. For the sake of brevity, we will prove only the inequality:
\[\max\limits_{q \in [n]} \twonorm{e_{q}^\top\left(H^\top H\right)^{a}V} \leq (\rho n \infnorm{H_2} + c \log n)^{2a} \tau\]
The rest of the lemma follows by applying similar arguments to the appropriate quantities.

Let $\omega:[2a] \rightarrow \{1,2\}$ be a function used to index a single term in the expansion of $(H^\top H)^a$. We express the term as follows:

\[
(H^\top H)^a = \sum\limits_{\omega} \prod\limits_{i = 1}^{a} H_{\omega(2i - 1)}^\top H_{\omega(2i)}
\]

We will now fix one such term $\omega$ and then bound the length of the following random vector:

\[
v_{\omega} = e_{1}^\top \prod\limits_{i = 1}^{a} (H^\top_{\omega(2i - 1)}H_{\omega(2i)}) V
\]

Let $\alpha$ be used to denote a tuple $(i,j)$ of integers used to index entries in a matrix. Let $T(i)$ be used to denote the parity function computed on $i$, i.e, $0$ if $i$ is divisible by $2$ and $1$ otherwise. This function indicates if the matrix in the expansion is transposed or not. We now introduce ${B_{(i,j),(k,l)}^{p,q}, \, p \in \{1,2\}, \, q \in \{0,1\}}$ and $A_{(i,j)}^p, \, p \in \{1, 2\}$ which are defined as follows:

\[
A_{(i,j)}^p \coloneqq \delta_{i,1}(\delta_{p,1} + \delta_{p,2} \ind{(i,j) \in Supp(H_{2})})
\]

\[
B_{(i,j), (k,l)}^{p,q} \coloneqq (\delta_{q,1}\delta_{j,l} + \delta_{q,0}\delta_{i,k}) (\delta_{p,1} + \delta_{p,2} \ind{(k,l) \in Supp(H_{2})})
\]

where $\delta_{i,j} = 1$ if $i=j$ and 0 otherwise. We will subsequently write the random vector $v_{\omega}$ in terms of the individual entries of the matrices. The role of $B_{(i,j), (k,l)}^{p,q}$ and $A_{(i,j)}^p$ is to ensure consistency in the terms used to describe $v_{\omega}$. We will use $h_{i, \alpha}$ to refer to $(H_{i})_{\alpha}$.

With this notation in hand, we are ready to describe $v_{\omega}$.

\[
v_{\omega} = \sum\limits_{\substack{\alpha_{1}, \dots, \alpha_{2a} \\ \alpha_{1}(1) = 1}} A_{\alpha_1}^{\omega(1)}B^{\omega(2), T(2)}_{\alpha_1\alpha_2} \dots B^{\omega(2a), T(2a)}_{\alpha_{2a - 1}\alpha_{2a}} h_{\omega(1), \alpha_{1}} \cdots h_{\omega(2a), \alpha_{2a}} v_{\alpha_{2a}(2)}
\]

We now write the squared length of $v_{\omega}$ as follows:
  \begin{multline*}
    X_{\omega} = \sum\limits_{\substack{\alpha_{1}, \dots, \alpha_{2a}, \alpha^{\prime}_{1}, \dots, \alpha^{\prime}_{2a} \\ \alpha_{1}(1) = 1, \alpha^{\prime}_{1}(1) = 1}} A_{\alpha_1}^{\omega(1)} B^{\omega(2), T(2)}_{\alpha_1\alpha_2} \dots B^{\omega(2a), T(2a)}_{\alpha_{2a - 1}\alpha_{2a}} h_{\omega(1), \alpha_{1}} \cdots h_{\omega(2a), \alpha_{2a}} \\
    A_{\alpha_1}^{\omega(1)} B^{\omega(2), T(2)}_{\alpha^{\prime}_1\alpha^{\prime}_2} \dots B^{\omega(2a), T(2a)}_{\alpha^{\prime}_{2a - 1}\alpha^{\prime}_{2a}} h_{\omega(1), \alpha^{\prime}_{1}} \cdots h_{\omega(2a), \alpha^{\prime}_{2a}} \langle v_{\alpha_{2a}(2)}, v_{\alpha^{\prime}_{2a}(2)} \rangle
  \end{multline*}

  We can see from the above equations that the entries used to represent $v_{\omega}$ are defined with respect to paths in a bipartite graph. In the following, we introduce notations to represent entire paths rather than just individual edges:

  Let $\boldsymbol{\alpha}\coloneqq(\alpha_1,\ldots,\alpha_{2a})$ and

  \[
  \bm{\zeta_{\bm{\alpha}}}\coloneqq A_{\alpha_1}^{\omega(1)} B^{\omega(2), T(2)}_{\alpha_1\alpha_2}\ldots B^{\omega(2a), T(2a)}_{\alpha_{2a-1}\alpha_{2a}} h_{\omega{(1)}, \alpha_1}\ldots h_{\omega{(2a)}, \alpha_{2a}}
  \]

  Now, we can write:
  \[
  X_\omega=\sum\limits_{\substack{\boldsymbol{\alpha}, \bm{\alpha^{\prime}} \\ \alpha_1(1) = \alpha'_1(1) = 1}}\zeta_{\boldsymbol{\alpha}}\zeta_{\boldsymbol{\alpha^{\prime}}} \langle v_{\alpha_{2a}(2)}, v_{\alpha^{\prime}_{2a}(2)} \rangle
  \]

  Calculating the $k^{\text{th}}$ moment expansion of $X_\omega$ for some number $k$, we obtain:

  \begin{equation}
    \label{k-mom}
    \mathbb{E}[X_\omega^k]=\sum\limits_{\boldsymbol{\alpha^1},\ldots,\boldsymbol{\alpha^{2k}}}\mathbb{E}[\zeta_{\boldsymbol{\alpha^1}}\ldots\zeta_{\boldsymbol{\alpha^{2k}}} \langle v_{\alpha^{1}_{2a}(2)}, v_{\alpha^{2}_{2a}(2)} \rangle \dots \langle v_{\alpha^{2k - 1}_{2a}(2)}, v_{\alpha^{2k}_{2a}(2)} \rangle]
  \end{equation}

We now show how to bound the above moment effectively. Notice that the moment is defined with respect to a collection of $2k$ paths. We denote this collection by $\Delta \coloneqq (\boldsymbol{\alpha^1},\ldots,\boldsymbol{\alpha^{2k}})$. For each such collection, we define a partition $\Gamma(\Delta)$ of the index set $\{(s, l): s \in [2k], l \in [2a]\}$ where $(s,l)$ and $(s',l')$ are in the same equivalence class if $\omega(l) = \omega(l') = 1$ and $ \alpha^s_l = \alpha^{s'}_{l'}$. Additionally, each $(s,l)$ such that $\omega(l) = 2$ is in a separate equivalence class.

We bound the expression in (\ref{k-mom}) by partitioning all possible collections of $2k$ paths based on the partitions defined by them in the above manner. We then proceed to bound the contribution of any one specific path to (\ref{k-mom}) following a particular partition $\Gamma$, the number of paths satisfying that particular partition and finally, the total number of partitions. Since, $H_{1}$ is a matrix with $0$ mean, any equivalence class containing an index $(s,l)$ such that $\omega(l) = 1$ contains at least two elements.

We proceed to bound (\ref{k-mom}) by taking absolute values:
\begin{equation}
    \label{eqn:kAbsMom}
    \mathbb{E}[X_\omega^k]\leq\sum\limits_{\boldsymbol{\alpha^1},\ldots,\boldsymbol{\alpha^{2k}}}\mathbb{E}[|\zeta_{\boldsymbol{\alpha^1}}|\ldots|\zeta_{\boldsymbol{\alpha^{2k}}}| |\langle v_{\alpha^{1}_{2a}(2)}, v_{\alpha^{2}_{2a}(2)} \rangle| \dots |\langle v_{\alpha^{2k - 1}_{2a}(2)}, v_{\alpha^{2k}_{2a}(2)} \rangle|]
\end{equation}

We now fix one particular partition and bound the contribution to (\ref{eqn:kAbsMom}) of all collections of paths $\Delta$ that correspond to a valid partition $\Gamma$.

We construct from $\Gamma$ a directed multigraph $G$. The equivalence classes of $\Gamma$ form the vertex set of G, $V(G)$. There are 4 kinds of edges in $G$ where each type is indexed by a tuple $(p,q)$ where $p \in \{1,2\}, \, q \in \{0,1\}$. We denote the edge sets corresponding to these 4 edge types by $E_{(1,0)}$, $E_{(1,1)}$, $E_{(2,0)}$ and $E_{(2,1)}$ respectively. An edge of type $(p,q)$ exists from equivalence class $\gamma_{1}$ to equivalence class $\gamma_{2}$ if there exists $(s,l) \in \gamma_{1}$ and $(s',l') \in \gamma_{2}$ such that $l' = l + 1$, $s = s'$, $\omega(s') = p$ and $T(l') = q$.

The summation in \ref{eqn:kAbsMom} can be written as follows:

  \begin{align*}
    \hspace{2em}&\hspace{-2em} \mathbb{E}[\abs{\zeta_{\boldsymbol{\alpha}^{1}}}\dots \abs{\zeta_{\boldsymbol{\alpha}^{2k}}} \abs{\langle v_{\alpha^{1}_{2a}(2)}, v_{\alpha^{2}_{2a}(2)} \rangle} \dots \abs{\langle v_{\alpha^{2k - 1}_{2a}(2)}, v_{\alpha^{2k}_{2a}(2)} \rangle}]
    \\
    &\leq \tau^{2k} \left( \prod\limits_{s = 1}^{2k} A_{\alpha^s_1}^{\omega(1)}\prod\limits_{l = 1}^{2a - 1} B^{\omega(l+1), T(l+1)}_{\alpha_{l}^{s}, \alpha_{l + 1}^{s}} \right) \mathbb{E} \left[ \left( \prod\limits_{s = 1}^{2k} \prod\limits_{l = 1}^{2a} \abs{h_{\omega(l), \alpha_{l}^{s}}} \right) \right]
    \\
    & \overset{(\zeta_1)}{\leq} \tau^{2k} \left ( \prod\limits_{s = 1}^{2k} A_{\alpha^s_1}^{\omega(1)}\prod\limits_{l = 1}^{2a - 1} B^{\omega(l+1), T(l+1)}_{\alpha_{l}^{s}, \alpha_{l + 1}^{s}} \right ) \prod\limits_{\gamma \in V_{1}(G)} \ddfrac{1}{n} \prod\limits_{\gamma \in V_{2}(G)} \norm{H_{2}}_{\infty}
    \\
    &= \ddfrac{\tau^{2k} \norm{H_{2}}_{\infty}^{w_{2}}}{n^{w_{1}}} \left ( \prod\limits_{s = 1}^{2k}A_{\alpha^s_1}^{\omega(1)}\prod\limits_{l = 1}^{2a - 1} B^{\omega(l+1), T(l+1)}_{\alpha_{l}^{s}, \alpha_{l + 1}^{s}} \right )
  \end{align*}

   where $(\zeta_1)$ follows from the moment conditions on $H_1$. $V_{1}(G)$ and $V_{2}(G)$ are the vertices in the graph corresponding to tuples $(i,j)$ such that $\omega(j) = 1$ and $\omega(j) = 2$ respectively and $w_{1} = \abs{V_{1}(G)}$, $w_{2} = \abs{V_{2}(G)}$.

   We first consider an equivalence class $\gamma_{1}$ such that there exists an index $(s,l) \in \gamma_{1}$ and $l = 1$. We form a spanning tree $T_{1}$ of all the nodes reachable from $\gamma_{1}$ with $\gamma_{1}$ as root. We then remove the nodes $V(T_{1})$ from the graph $G$ and repeat this procedure until we obtain a set of $l$ trees $T_{1},\dots,T_{l}$ with roots $\gamma_{1}, \dots, \gamma_{l}$ such that $\bigcup\limits_{i = 1}^{l} V(G_{i}) = V(G)$. This happens because every node is reachable from some equivalence class which contains an index of the form $(s, 1)$. Also, each of these trees $T_{i}, \ \forall \, i \in [l]$ is disjoint in their vertex sets. Given this decomposition, we can factorize the above product as follows:

 \begin{multline}
 \label{eqn:treeEqn}
     \mathbb{E}[X_\omega^k] \leq \ddfrac{\tau^{2k} \norm{H_{2}}_{\infty}^{w_{2}}}{n^{w_{1}}} \prod\limits_{j = 1}^{l} \sum\limits_{\alpha_1, \dots, \alpha_{v_j}} A_{\alpha_{1}}^{\omega(1)} \prod\limits_{\{\gamma, \gamma'\} \in E_{(1,0)}(T_{j})} B^{1, 0}_{\alpha_{\gamma}\alpha_{\gamma'}} \\ \prod\limits_{\{\gamma, \gamma'\} \in E_{(1,1)}(T_{j})} B^{1, 1}_{\alpha_{\gamma}\alpha_{\gamma'}} \prod\limits_{\{\gamma, \gamma'\} \in E_{(2,0)}(T_{j})} B^{2, 0}_{\alpha_{\gamma}\alpha_{\gamma'}} \prod\limits_{\{\gamma, \gamma'\} \in E_{(2,1)}(T_{j})} B^{2, 1}_{\alpha_{\gamma}\alpha_{\gamma'}}
 \end{multline}

 For a single connected component, we can compute the summation bottom up from the leaves. First, notice that:

 \begin{center}
 \begin{tabular}{c c}
      $\sum\limits_{\alpha_{\gamma'}} B^{2, 1}_{\alpha_{\gamma}\alpha_{\gamma'}} \leq \rho n$ & $\sum\limits_{\alpha_{\gamma'}} B^{2, 0}_{\alpha_{\gamma}\alpha_{\gamma'}} \leq \rho n$ \\
      $\sum\limits_{\alpha_{\gamma'}} B^{1, 1}_{\alpha_{\gamma}\alpha_{\gamma'}} = n$ &
      $\sum\limits_{\alpha_{\gamma'}} B^{1, 0}_{\alpha_{\gamma}\alpha_{\gamma'}} = n$
 \end{tabular}
 \end{center}

 Where the first two follow from the sparsity of $H_{2}$. Every node in the tree $T_{j}$ with the exception of the root has a single incoming edge. For the root, $\gamma_{j}$, we have:

 \begin{center}
 \begin{tabular}{c c}
      $\sum\limits_{\alpha_{1}} A_{\alpha_{1}}^{\omega(1)} \leq \rho n \text{ for } \omega(1) = 2$
      &
      $\sum\limits_{\alpha_{1}} A_{\alpha_{1}}^{\omega(1)} = n \text{ for } \omega(1) = 1$
 \end{tabular}
 \end{center}

 From the above two observations, we have:

 \begin{multline*}
     \sum\limits_{\alpha_1, \dots, \alpha_{v_j}} A_{\alpha_{1}}^{\omega(1)} \prod\limits_{\{\gamma, \gamma'\} \in E_{(1,0)}(T_{j})} B^{1, 0}_{\alpha_{\gamma}\alpha_{\gamma'}}  \prod\limits_{\{\gamma, \gamma'\} \in E_{(1,1)}(T_{j})} B^{1, 1}_{\alpha_{\gamma}\alpha_{\gamma'}} \prod\limits_{\{\gamma, \gamma'\} \in E_{(2,0)}(T_{j})} B^{2, 0}_{\alpha_{\gamma}\alpha_{\gamma'}} \\
     \prod\limits_{\{\gamma, \gamma'\} \in E_{(2,1)}(T_{j})} B^{2, 1}_{\alpha_{\gamma}\alpha_{\gamma'}} \leq (\rho n)^{w_{2,j}}n^{w_{1,j}}
 \end{multline*}

  where $w_{k,j}$ represents the number of vertices in the $j^{th}$ component which contain tuples $(i,j)$ such that $\omega(j) = k$ for $k \in \{1,2\}$.

  Plugging the above in (\ref{eqn:treeEqn}) gives us

  \[
  \mathbb{E}[X_\omega^k(\Gamma)] \leq \ddfrac{\tau^{2k}\norm{H_2}_\infty^{w_{2}}}{n^{w_1}}(\rho n)^{\sum_j w_{2,j}}n^{\sum_j w_{1,j}}=\tau^{2k}\norm{H_2}^{w_{2}}_\infty (\rho n)^{w_2}
  \]

  Let $a_{1}$ and $a_{2}$ be defined as $\abs{\{i: \omega(i) = 1\}}$ and $\abs{\{i: \omega(i) = 2\}}$ respectively (Note that $w_{2} = 2a_{2}k$). Summing up over all possible partitions (there are $(2a_{1}k)^{2a_{1}k}$ of them), we get our final bound on $\mathbb{E}\left[\hat{X}_{\omega}^{k}\right]$ as $\tau^{2k}(\rho n \norm{H_{2}}_{\infty})^{2a_{2}k}(2a_{1}k)^{2a_{1}k}$.

  Now, we bound the probability that $\hat{X}_{\omega}$ is too large. Choosing $k = \left\lceil \frac{\log n}{a_{1}} \right\rceil$ and applying the $k^{th}$ moment Markov inequality, we obtain:

  \begin{align*}
    \mathbf{Pr}\left[\abs{\hat{X}_{\omega}} > (c\log n)^{2a_{1}}\tau^{2}(\rho n \norm{H_{2}}_{\infty})^{2a_{2}}\right] &\leq \mathbb{E} \left[\abs{\hat{X}_{\omega}}^{k} \right] \left(\ddfrac{1}{(c\log n)^{2a_{1}}\tau^{2}(\rho n \norm{H_{2}}_{\infty})^{2a_{2}}} \right)^{k} \\
    &\leq \left(\ddfrac{2ka_{1}}{c\log{n}}\right)^{2ka_{1}} \\
    &\leq n^{-2\log{\frac{c}{4}}}
  \end{align*}

  Taking a union bound over all the $2^a$ possible $\omega$, over values of $a$ from $1$ to $\log n$ and over the $n$ values of $q$, we get the required result.
\end{proof}

\subsection{Additional Experimental Results}
\label{app:exp}

\begin{figure}[t]
\centering
\begin{tabular}{cccc}
\hspace*{-10pt}\includegraphics[width=0.5\textwidth]{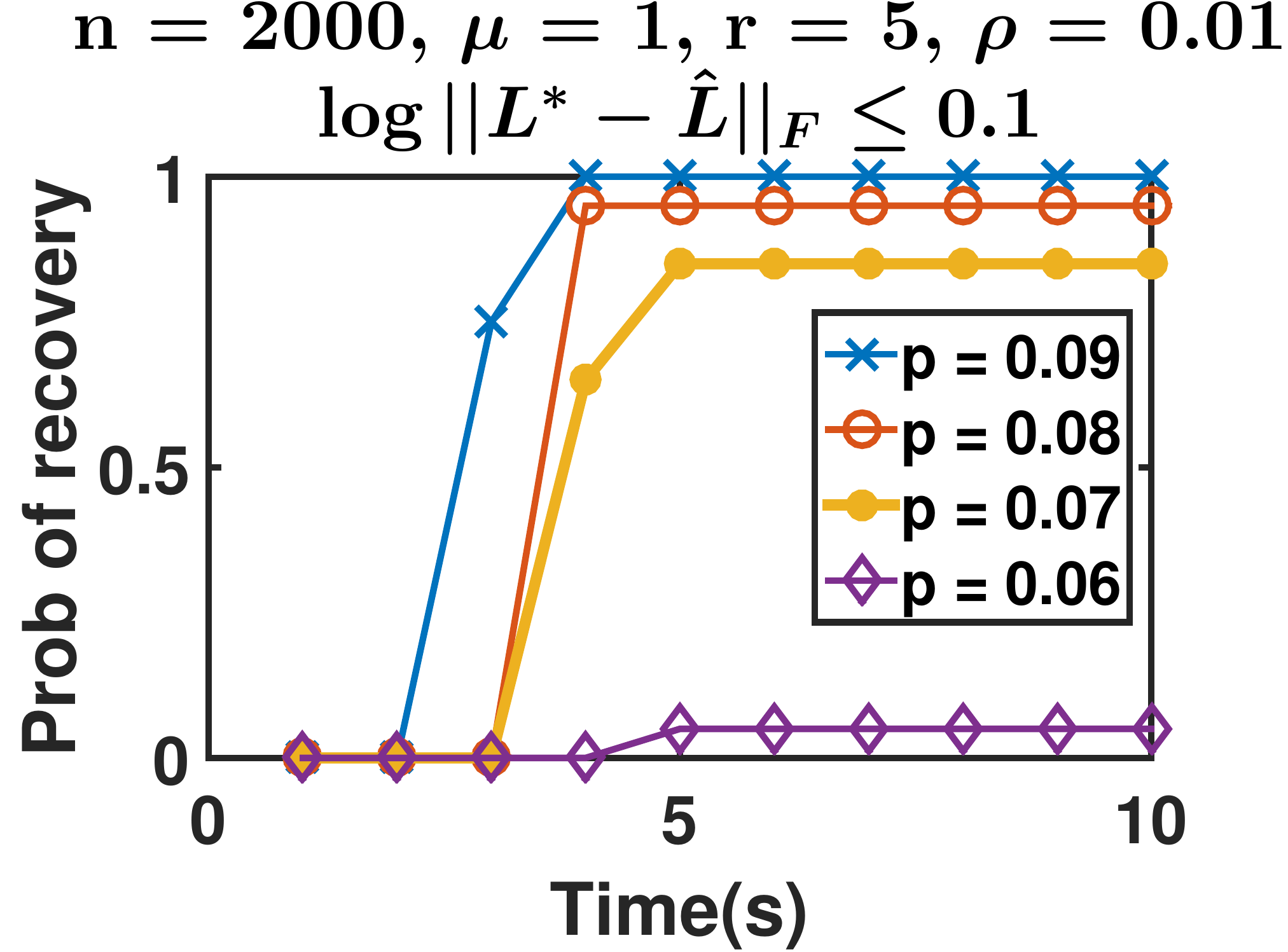}&
\hspace*{-10pt}\includegraphics[width=0.5\textwidth]{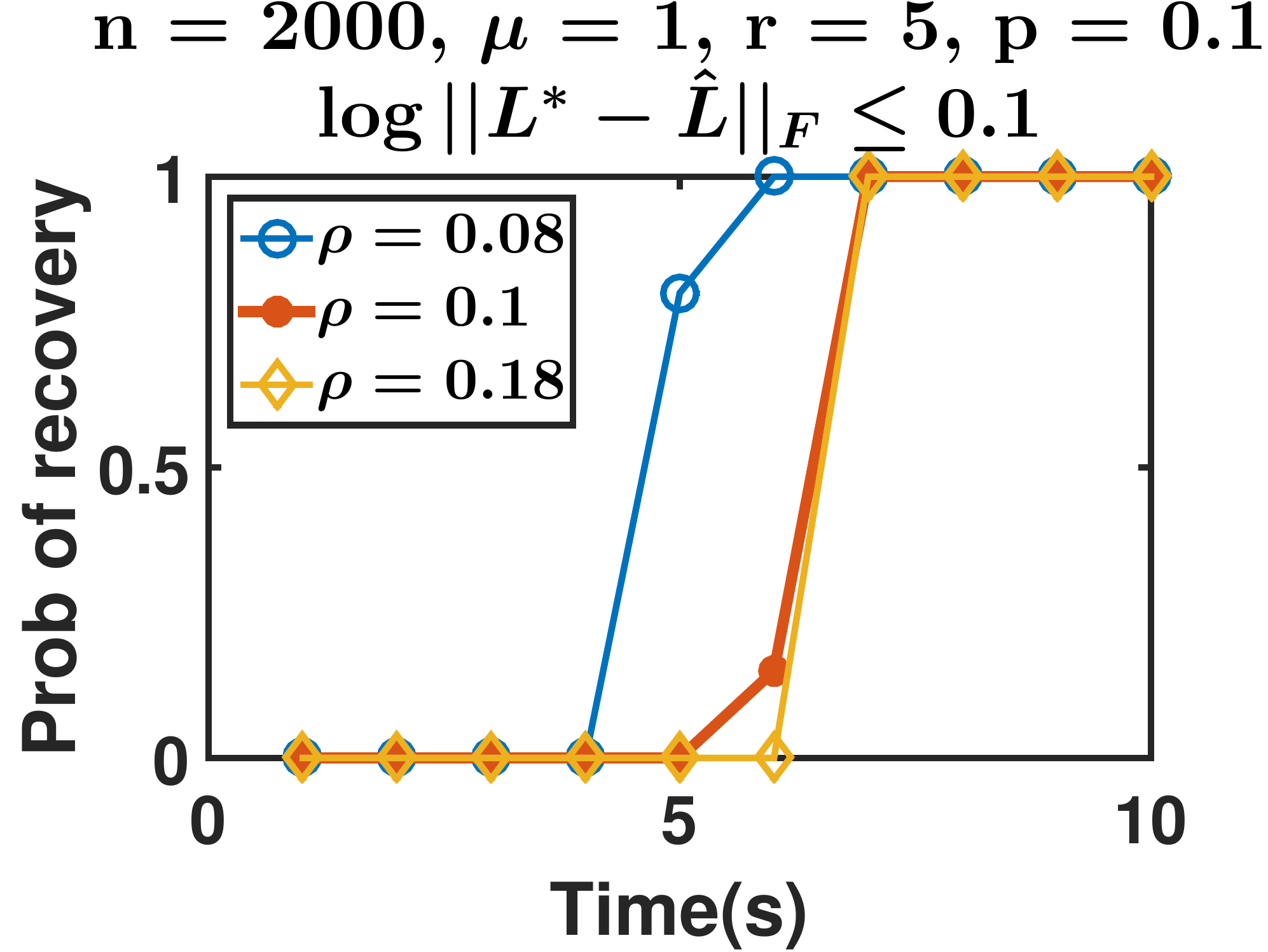}\\[-3pt]
(a)&(b)\\
\end{tabular}\vspace*{-8pt}
\caption{We run the \ncgca algorithm with extremal values of sampling probability and fraction of corruptions, and record the probability with which we recover the original matrix, (a) : time vs probability of recovery for very small values of sampling probability, (b) : time vs probability of recovery for large number of corruptions ($\rho n^2$)}
\label{fig:synth_plots_1_app}
\end{figure}

We detail some additional experiments performed with Algorithm \ref{alg:gca} in this section. The experiments were performed on synthetic data and real world data sets.

\textbf{Synthetic data}. We generate a random matrix $M\in\R^{2000\times2000}$ in the same way as described in Section \ref{sec:exp}. In these experiments our aim is to analyze the behavior of the algorithm in extremal cases. We consider two of such cases : 1) sampling probability is very low (Figure \ref{fig:synth_plots_1_app} (a)), 2) number of corruptions is very large (Figure \ref{fig:synth_plots_1_app} (b)). In the first case, we see that the we get a reasonably good probability of recovery $(\sim 0.8)$ even with very low sampling probability $(0.07)$. In the second case, we observe that the time taken to recover seems almost independent of the number of corruptions as long as they are below a certain threshold. In our experiments we saw that on increasing the $\rho$ to 0.2 the probability of recovery went to 0. To compute the probability of recovery we ran the experiment 20 times and counted the number of successful runs.

\begin{figure}[t]
\centering
\begin{tabular}{ccc}
\includegraphics[width=0.33\textwidth]{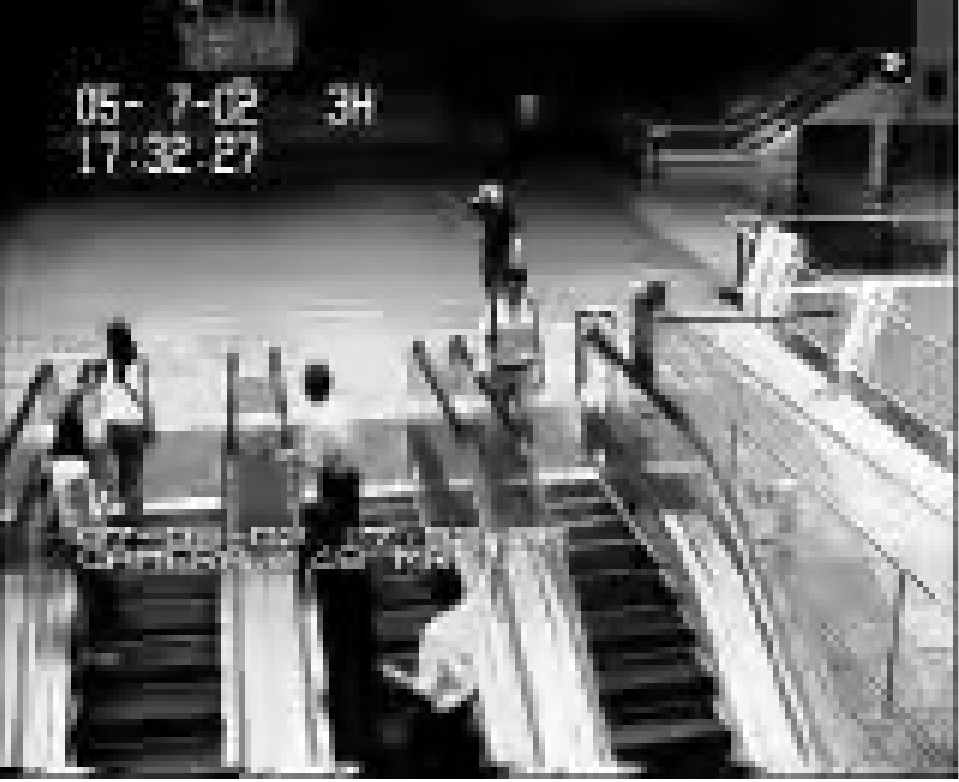}&
\includegraphics[width=0.33\textwidth]{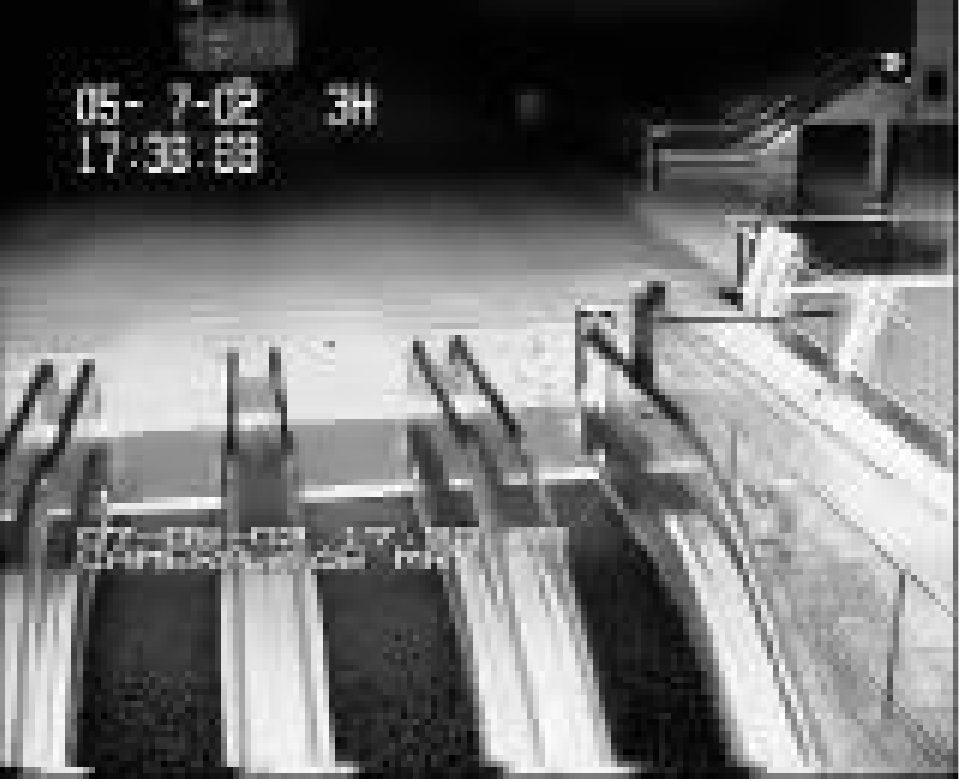}&
\includegraphics[width=0.33\textwidth]{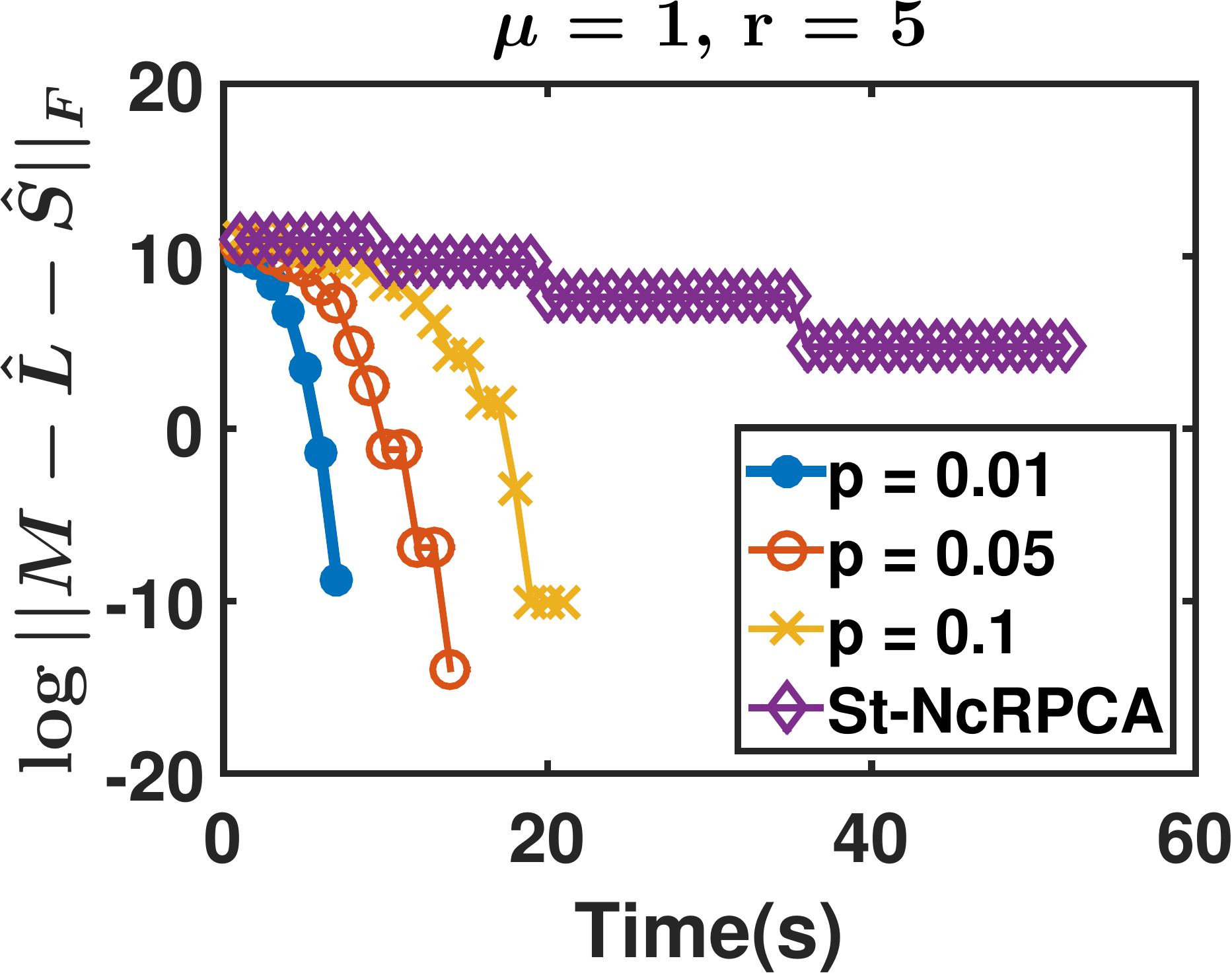}\\
(a)&(b)&(c)\\
\end{tabular}\vspace*{-8pt}
\caption{\ncgca on Escalator video. (a): a video frame (b): an extracted background frame (c): time vs error for different sampling probabilities; \ncgca takes 7.3s while St-NcRMC takes 52.9s}
\label{fig:real_plots_1_app}
\end{figure}

\textbf{Foreground-background separation.} We present results for one more real world data set in this section. We applied our \ncgca method (with varying $p$) to the Escalator video. Figure~\ref{fig:real_plots_1_app} (a) shows one frame from the video. Figure~\ref{fig:real_plots_1_app} (b) shows the extracted background from the video by using our method (\ncgca, Algorithm~\ref{alg:gca}) with probability of sampling $p=0.05$. Figure~\ref{fig:real_plots_1_app} (c) compares objective function value for different $p$ values.

\end{document}